\documentclass[twoside,11pt, preprint]{article}

\usepackage{blindtext}
\usepackage{jmlr2e}
\setcitestyle{numbers,square}
\usepackage{subcaption}
\usepackage{wrapfig}
\usepackage{amsmath}
\usepackage{amsfonts}
\usepackage[nameinlink,noabbrev]{cleveref}
\usepackage{bm}

\usepackage{tikz}
\usepackage{tikz-cd}
\usepackage{pgfplots}

\pgfplotsset{compat=1.18}

\usetikzlibrary{positioning}

\newtheorem{assumption}{Assumption}
\crefname{assumption}{assumption}{assumptions}
\Crefname{assumption}{Assumption}{Assumptions}

\newcommand{\neuronin}{{\bm{w}\cdot \bm{x}}}

\newcommand{\defeq}{\overset{\text{\tiny def}}{=}}
\newcommand{\norm}[1]{\left\lVert #1\right\rVert}
\newcommand{\ip}[2]{\left\langle #1, #2 \right\rangle}
\newcommand{\abs}[1]{\left\lvert #1 \right\rvert}
\newcommand{\SQdim}{{SQ\dim}}
\newcommand{\threshold}[1]{\sigma^{\pm}(#1)}
\DeclareMathOperator{\sech}{sech}

\newcommand{\thmref}[1]{Thm.~\ref{#1}}
\newcommand{\propref}[1]{Prop.~\ref{#1}}

\newcommand{\corref}[1]{Corollary~\ref{#1}}
\newcommand{\defref}[1]{Definition~\ref{#1}}

\newlength{\figurewidth}
\pgfmathsetmacro{\uppergap}{1.40}
\usepackage{lastpage}

\firstpageno{1}

\begin{document}

\title{
    Representation Complexity Along Learning Trajectories: Trading Width for Precision and Time
}

\author{%
    \name Ryan Singh \email rs773@sussex.ac.uk \\ 
    \addr Department of Engineering and Informatics\\ University of Sussex 
    \AND \name Miguel De Llanza \email {} \\ 
    \addr Department of Engineering and Informatics\\ University of Sussex 
    \AND 
    \name Christopher L. Buckley \email {} \\ \addr Department of Engineering and Informatics\\ University of Sussex }

\maketitle

\begin{abstract}%
What computational role does overparameterisation play during learning? Information-plane analyses suggest that neural networks can traverse representations substantially more complex than those required by their final predictors. We explore when this transient complexity serves a computational purpose by comparing the representation complexity needed for prediction with that needed to make an unknown target statistically visible through gradient feedback. We introduce an oracle model of neural network learning and use Kolmogorov \(n\)-widths to quantify effective representation complexity. Our main technical result transfers complexity bounds from the forward representation to the corresponding gradient-query features. Combined with statistical-query lower bounds, this gives a task-agnostic lower bound on the initial effective width required for target-dependent gradient information to become detectable. 
For \(k\)-sparse parity in \(D\) dimensions, a known target admits an \(O(k)\)-width representation, while the general bound already implies a polynomial separation between representation and discovery. We then exploit parity's invariance under coordinatewise sign reflections to identify a sharper unsigned effective width, which removes sign distinctions invisible to the parity information channel. At a transfer-compatible measurement resolution, its critical scale is \(
    \widetilde{\Theta}_k
    \left(
        D^{k-1}\tau^2
    \right)
\), up to logarithmic factors.
Thus, the representation resources required to discover a predictor can substantially exceed those required to represent it once learned. Interpreting precision through accumulation over fresh samples yields a corresponding width-time frontier. Experiments support the predicted dimension dependence and examine how optimisation and generalisation depart from this idealised frontier.
\end{abstract}


\section{Introduction}
Modern machine learning often uses highly overparameterised models whose parameters are statistically under-determined by the available data \citep{zhangUnderstandingDeepLearning2021}. Learning theory offers several explanations for how such models can nevertheless generalise through inductive biases that favour low-complexity solutions \citep{belkinReconcilingModernMachinelearning2019, gunasekarCharacterizingImplicitBias2018, woodworthKernelRichRegimes2020, azulayImplicitBiasInitialization2021}. For neural networks, such biases are best understood in analytically tractable settings such as the lazy or kernel regime \citep{jacotNeuralTangentKernel2018,
woodworthKernelRichRegimes2020}, where the representation remains
effectively fixed and gradient descent selects a controlled-RKHS-norm
solution. Outside this regime, the representation evolves with the parameters \cite{meiMeanfieldTheoryTwolayers2019, chizatImplicitBiasGradient2020}, making both the optimisation dynamics and their implicit biases substantially harder to characterise.

Feature learning raises a different complexity question. The final predictor may admit a compact representation even though discovering it requires the network to pass through a substantially more diverse one. Empirical information-plane studies have reported non-monotonic representation trajectories, in which complexity increases before decreasing as training approaches its final solution \cite{shwartz-zivOpeningBlackBox2017, saxeInformationBottleneckTheory2019}. This motivates a basic question: can a target be simple to represent but \textit{require} a more complex representation to discover?

Our starting point is that a representation sufficient to express a known target need not expose enough target-dependent information to discover an unknown one. Existing complexity measures primarily characterise the former. We introduce effective width as a functional measure that also constrains the information available through gradient feedback. For sparse parity, the target is compact once known but requires polynomially greater representation complexity to become visible to the optimiser. Effective width therefore places representation diversity, feedback precision, and optimisation time on a common frontier, giving overparameterisation a computational role that may disappear once learning is complete.

\subsection{Modelling Approach}
Rather than analysing a particular optimisation algorithm, we study an oracle model of neural network training. Specifically, we consider a two-layer neural network 
\[ 
f(\,\cdot\,;\theta): \mathbb R^D \longrightarrow \mathbb R, \qquad f(\bm x;\theta) = \sum_{p=1}^{P} a_p\sigma(\bm w_p\cdot\bm x). 
\]
The optimiser does not observe the target distribution directly. Instead, it receives coordinatewise noisy gradient information through an oracle,  
\[\hat{g}_{p,d}^t =\mathbb{E}[x_d\sigma'(\bm{w_p}^t \cdot \bm{x}) y] +\xi_{p,d}^t, \quad \xi_{p, d}^t \in [-\tau, \tau].\]
In particular we focus on zero-respecting learners, which cannot change their first-layer representation under an all-zero feedback matrix: 
\[
    W^{t+1} = W^{t} + \phi(W^t, \widehat{G}^t), \quad \text{where} \quad \phi(W, 0_{P\times D})= 0_{P\times D}.
\]
The model is deliberately broader than a direct description of gradient
descent. Explicit second-layer adaptation and other optimiser-specific
state could be added without changing the lower-bound calculation
(see \cref{app:model}). Importantly, the direct target-dependent readout signal
\(\mathbb E[z_p(\bm x)y]\) is already constrained by the width of the forward
activities. We therefore omit these
additional variables, which affect how the exposed information is used
rather than whether it is available.

\subsection{Summary of Results}
Architectural width alone does not measure the functional complexity available to learning, since many neurons may realise identical or nearly identical activities. We therefore use Kolmogorov \(n\)-widths \citep{kolmogoroffUberBesteAnnaherung1936, pinkusNWidthsApproximationTheory2012} in Variation space (\defref{def:Barron}) to define \( R_{\mathcal B_z}(\mathcal Z,\delta), \) the effective width of the activity set \(\mathcal Z\) at approximation resolution \(\delta\).

\paragraph{Effective width bound.} The same first-layer parameters determine   both the activities used for prediction and the gradient queries available for learning. We exploit this shared geometry (\propref{prop:main-width-transfer}) to transfer statistical-query lower bounds into lower bounds on effective width. For a concept class of statistical-query dimension \(d_{\mathrm{SQ}}\) \cite{reyzinStatisticalQueriesStatistical2020}, any successful zero-respecting learner requires, up to approximation factors,
\[
\boxed{
    R_{\mathcal B_z}(\mathcal Z_0,\delta)
    =
    \Omega\left(
        \frac{d_{\mathrm{SQ}}}{D^2}\tau^2
    \right)
    }
\]
where \(\mathcal Z^0\) is the initial activity set and \(d_{\mathrm{SQ}}\) is the statistical-query dimension of the concept class (\thmref{thm:critical-initial-width}). Below this threshold, distinct targets admit the same all-zero gradient transcript, so the algorithm cannot make target-dependent progress.

\paragraph{Example: Sparse Parities.} 
Consider the problem of learning a $k$-sparse parity under independent Rademacher inputs:
\[
    \mathcal{H}_k = \{\chi_S: S \subset [D], \lvert S \rvert = k\}, \quad \bm{x} \sim \{\pm 1 \}^{D}, \quad  \chi_S(\bm{x}) = \prod_{i \in S} x_i.
\]
Once \(S\) is known, \(\chi_S\) admits an \(O(k)\)-width \(\tanh\)-network representation, independently of \(D\). Discovering \(S\) is substantially harder. Since
\(\SQdim(\mathcal H_k,0) = \binom Dk \) \citep{barakHiddenProgressDeep2022}, 
our general lower bound on the critical initial effective width gives,
\[ 
    \boxed{ R_{\mathrm{critical}} = \Omega_k    \left( D^{k-2}\tau^2 \right) } 
\] 
at the transfer-compatible resolution. Thus, for fixed $k>2$ and non-vanishing tolerance a target can be simple to represent once known while requiring polynomially greater effective width to become visible through noisy gradient feedback.

\paragraph{Parity symmetry and near-matching unsigned bounds} While the presented lower bound is general with respect to the concept class, parity problems exhibit a symmetry under sign transformations. This allows us to tighten the bound by measuring the unsigned width of the neurons:
\[
    R^+\defeq R_{\mathcal B_z}(\mathcal{Z}^+, \delta), \quad \text{where}\quad \mathcal{Z}^{+}(W) = \{ \sigma(\abs{\bm{w}}\cdot x): w \in W \}.
\]
The corresponding parity-relevant transfer incurs only a factor \(D\), rather than \(D^2\).
Building on the majority-neuron construction of
\citet{barakHiddenProgressDeep2022}, we obtain one-round upper bounds
that interpolate between constant width at
\(\tau^2=O_k(D^{-(k-1)})\) and
\(O_k(D^{k-1})\) width at constant precision. More generally, if
\(\tau^2=\Theta(D^{-\alpha})\), with \(0<\alpha<k-1\), then
\(O_k(D^{k-1-\alpha}\log D)\) pooled neurons suffice
(\thmref{thm:one-round-parity-upper}). Consequently, the critical parity-relevant effective width is characterised, up to logarithmic factors, by 
\[ \boxed{ R_{\mathrm{critical}}^{+} = \widetilde\Theta_k \left( D^{k-1}\tau^2 \right). } \]

\subsection{Statistical Interpretations of Effective Width}
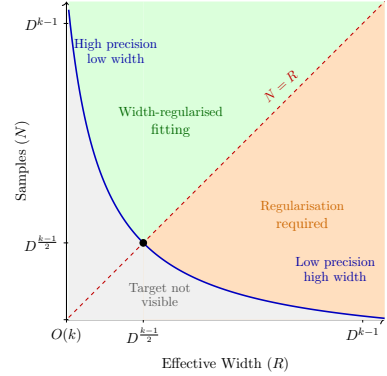
\begin{wrapfigure}{r}{0.48\textwidth}
        \centering
    \begin{tikzpicture}[ scale=\figurewidth, transform shape ]
\begin{axis}[
            width=10cm    ,
            height=10cm,
            xmin=0.3,
            xmax=3.2,
            ymin=0.3,
            ymax=3.2,
            axis lines=left,
            axis line style={->, thick},
            xlabel={
                Effective Width (\( R\))
            },
            ylabel={
                Samples (\(N\))
            },
            xlabel style={at={(axis description cs:0.5,-0.10)}, anchor=north},
            ylabel style={at={(axis description cs:-0.10,0.5)}, anchor=south},
            xtick={0.3, 1, 3},
            ytick={0.3, 1, 3},
            xticklabels={\(O(k)\), \(D^{\frac{k-1}{2}}\), \(D^{k-1}\)},
            yticklabels={, \(D^{\frac{k-1}{2}}\), \(D^{k-1}\)},
            tick style={black},
            clip=true,
        ]

            \addplot[
                draw=none,
                fill=gray!12,
                domain=0.32:3.2,
                samples=200,
            ]
                {min(1/x,3.2)}
                \closedcycle;

            \addplot[
                draw=none,
                fill=orange!25,
                domain=1:3.2,
                samples=150,
            ]
                {x}
                \closedcycle;

            \addplot[
                draw=none,
                fill=gray!12,
                domain=1:3.2,
                samples=150,
            ]
                {1/x}
                \closedcycle;

            \addplot[
                draw=none,
                fill=green!14,
                domain=0.32:1,
                samples=120,
            ]
                {3.2}
                \closedcycle;

            \addplot[
                draw=none,
                fill=white,
                domain=0.32:1,
                samples=120,
            ]
                {1/x}
                \closedcycle;

            \addplot[
                draw=none,
                fill=green!14,
                domain=1:3.2,
                samples=120,
            ]
                {3.2}
                \closedcycle;

            \addplot[
                draw=none,
                fill=white,
                domain=1:3.2,
                samples=120,
            ]
                {x}
                \closedcycle;

            \addplot[
                draw=none,
                fill=gray!12,
                domain=0.32:3.2,
                samples=200,
            ]
                {min(1/x,3.2)}
                \closedcycle;

            \addplot[
                draw=none,
                fill=orange!25,
                domain=1:3.2,
                samples=150,
            ]
                {x}
                \closedcycle;

            \addplot[
                draw=none,
                fill=gray!12,
                domain=1:3.2,
                samples=150,
            ]
                {1/x}
                \closedcycle;

            \pgfmathsetmacro{\uppergap}{1.40}
            
            

            \addplot[
                very thick,
                blue!75!black,
                domain=0.32:3.2,
                samples=200,
            ]
                {1/x};

            \addplot[
                thick,
                dashed,
                red!75!black,
                domain=0.32:3.2,
                samples=2,
            ]
                {x};

            \addplot[
                only marks,
                mark=*,
                mark size=2.5pt,
                black,
            ]
                coordinates {(1,1)};


            \node[
                red!75!black,
                inner sep=2pt,
                rotate=42,
            ]
                at (axis cs:2.25,2.43)
                {\small \(N=R\)};

            \node[
                align=center,
                blue!65!black,
            ]
                at (axis cs:0.75,2.75)
                {\small High precision\\[-1mm]
                 \small low width};

            \node[
                align=center,
                blue!65!black,
            ]
                at (axis cs:2.75,0.75)
                {\small Low precision\\[-1mm]
                 \small high width};

            \node[
                align=center,
                gray!70!black,
            ]
                at (axis cs:1.15,0.53)
                {\small Target not\\[-1mm]
                 \small visible};

            \node[
                align=center,
                green!40!black,
            ]
                at (axis cs:1.25,2.1)
                {\small Width-regularised \\ fitting};

            \node[
                align=center,
                orange!80!black,
            ]
                at (axis cs:2.45,1.25)
                {\small Regularisation \\ required};


        \end{axis}
\end{tikzpicture}

    \caption{
        \small Effective-width--sample phase diagram for \(k\)-sparse parity.
        The blue curve is the task-visibility frontier
        \(NR_{\mathrm{task}}(N)
        =\widetilde\Theta_k(D^{k-1})\). Above the frontier the target is detectable. The dashed red line $R=N$ marks the crossover in fixed-feature behaviour. When \(R<N\), exact
        interpolation requires a large readout norm; effective-width control consequently obstructs premature fitting. When \(R>N\) additional readout
        regularisation or implicit norm control may be required.
    }
    \label{fig:precision-diversity-frontier}
\end{wrapfigure} 
Here we specialise to nonnegative-weight initialisations, where
\(\mathcal Z=\mathcal Z^+\) and the widths governing parity discovery
and generalisation coincide. Under the fresh-sample approximation \(\tau^2\asymp N^{-1}\), the parity bounds at the  resolution \(\delta_N\asymp\frac{1}{N\sqrt D}\) yield the task frontier
 \[ 
 \boxed{
    NR_{\mathrm{task}}(N)=\widetilde\Theta_k(D^{k-1}).
    }
 \]
This relation provides a complexity--time interpretation of the bounds. A narrow representation requires greater statistical precision before the target becomes visible, whereas a more diverse representation exposes the target at coarser precision. 

\paragraph{Fixed-feature implications.} Understanding whether typical training can approach this frontier also requires accounting for fixed-feature fitting: a sufficiently rich representation may fit the observed sample through its readout before first-layer feature discovery is complete. By augmenting the width--time picture with a third coordinate, \(A\defeq\norm{\bm a}_1\), our secondary results show how effective width and sample size also constrain fixed-feature generalisation and interpolation.

Combining effective dimensionality with global variation-norm control
gives, schematically,
\[ 
    \mathcal L(f) - \widehat{\mathcal L}_N(f) \lesssim A \left[ (1+\delta) \sqrt{ \frac{ \min\{R,C_w^2D^2\} }N } + \delta \right] + \sqrt{ \frac{\log(1/\eta)}N }, 
\] 
where fixed activation- and loss-dependent constants are suppressed (\propref{prop:width-generalisation}). For sparse parity, width and readout norm are coupled by interpolation. When the representation is target independent and \(R/N\to0\), exact interpolation requires 
\[ 
    \norm{\bm a}_1 \geq \frac{1-o(1)}{\delta_N} 
\] 
with high probability (\propref{prop:parity-interpolation-cost}). 

Thus, when \(R\ll N\),   fitting must amplify weak residual features: the resulting norm inflation makes the preceding population certificate uninformative, while also obstructing premature empirical fitting. As \(R\) approaches \(N\), this obstruction can disappear because the dominant feature space may itself support interpolation. Effective-width control can therefore protect against premature fitting below the crossover. Once \(R\gtrsim N\), however, controlling the selected solution requires additional readout regularisation or an implicit optimisation bias toward small norm.

\subsection{Related Work}
\textbf{Gradient-based learning of parities.}
Classical statistical-query lower bounds explain the difficulty of
learning parity functions from approximate population correlations
\citep{blumWeaklyLearningDNF1994, kearnsEfficientNoisetolerantLearning1998, feldmanGeneralCharacterizationStatistical2017}. A growing literature studies when gradient-based algorithms can learn parity functions \citep{abbePolytimeUniversalityLimitations2020, danielyLearningParitiesNeural2020, barakHiddenProgressDeep2022, abbeLEARNINGHIGHDEGREEPARITIES2025, shoshaniHardnessLearningFixed2025}.
As prototypical computationally difficult targets, parities have also
been used to demonstrate separations between feature learning and
fixed-feature methods
\citep{malachQuantifyingBenefitUsing2021, shiTheoreticalAnalysisFeature2022}. Our work builds most directly on \citet{barakHiddenProgressDeep2022}, who study sparse parity under the
uniform input distribution. They identify a Fourier gap through which
gradient descent can gradually amplify the relevant coordinates despite
a long plateau in observable performance, and give polynomial-time
learning guarantees for networks whose width is independent of the
ambient dimension. Our focus is complementary: we connect the SQ obstruction to the
functional geometry of the current representation, study how effective
width trades against gradient precision, and show that functional
diversity rather than raw neuron count is the relevant resource.

\textbf{Neural variation spaces and Barron-type complexity.} The seminal work of \citet{BarronUniversalApproximationBounds1993} showed that two-layer neural networks can evade the curse of dimensionality for suitable classes of functions. One such class is characterised by bounded variation norm with respect to a neural feature dictionary \citep{ kurkovaDimensionIndependentRatesApproximation1997, kurkovaBoundsRatesVariablebasis2001, bachBreakingCurseDimensionality2017}. More recent work has studied the approximation, estimation, and structural properties of closely related Barron spaces \citep{ ePrioriEstimatesPopulation2019, eKolmogorovWidthDecay2021, eBarronSpaceFlowinduced2021, e.RepresentationFormulasPointwise2022}. We use this variation geometry to measure the finite-dimensional diversity of a representation and transfer its forward approximation complexity to the gradient features available during learning.

\textbf{Representation dynamics and feature learning.}
Recent work has identified several qualitative descriptions of
feature-learning dynamics, including saddle-to-saddle progression
through solutions of increasing complexity
\citep{jacotSaddletoSaddleDynamicsDeep2022, abbeSGDLearningNeural2023, zhangSaddletosaddleDynamicsExplains2026},
transitions between lazy and rich regimes
\citep{woodworthKernelRichRegimes2020, bordelonDynamicsFiniteWidth2023, tuMixedDynamicsLinear2024},
and improvements in the kernel induced by early feature-learning updates
\citep{baHighdimensionalAsymptoticsFeature2022}.
Motivated by empirical information-plane analyses
\citep{shwartz-zivOpeningBlackBox2017, saxeInformationBottleneckTheory2019},
we introduce a functional measure of representation dynamics. By controlling both forward approximation and backward gradient information, effective width gives representation dynamics an operational meaning. Together with optimisation time, it provides macroscopic coordinates for studying how departures from lazy dynamics support feature discovery.

\section{Main Results}
We now make our notion of representation complexity precise. We use the Kolmogorov \(n\)-width \cite{kolmogoroffUberBesteAnnaherung1936} to measure the error of the best n-dimensional linear approximation to a collection of representations.

\begin{definition}[Kolmogorov \(n\)-Width] 
\label{def:k-width}
Given normed space $X$ and $W\subset X$ the Kolmogorov \(n\)-width :
\[
    d_n(W, X) = \inf_{V\subset X:\dim(V)\leq n}\sup_{w \in W}\inf_{v \in V} \lvert \lvert v - w \rvert \rvert_X.
\]
\end{definition}
We define the effective width of $W$ at approximation resolution $\delta$ by
\[
    R_{X}(W, \delta) = \min \{n: d_n(W, X) \leq \delta\}.
\]
Thus $R_X(W,\delta)$ measures the number of linear degrees of freedom required to represent $W$ to accuracy $\delta$.

\subsection{A Forward-Backward Width Transfer}
The forward representation and backward information channel are coupled
through parameter differentiation. A direct function-space transfer is unavailable because the correspondence between the two neural dictionaries does not extend linearly to realised activity functions. Indeed, for a differentiable odd
activation,
\[
    z_{\bm w}(\bm x)
    \defeq
    \sigma(\bm w\cdot\bm x), \ 
    q_{\bm w,i}(\bm x)
    \defeq
    x_i\sigma'(\bm w\cdot\bm x) \implies
    z_{-\bm w}
    =
    -z_{\bm w}, \ 
    q_{-\bm w,i}
    =
    q_{\bm w,i},
\]
because \(\sigma'\) is even. Hence the atomwise correspondence cannot extend to a linear operator on realised activity functions. We therefore measure forward approximation in a Barron-type variation geometry
\citep{BarronUniversalApproximationBounds1993, bachBreakingCurseDimensionality2017, ePrioriEstimatesPopulation2019, e.RepresentationFormulasPointwise2022},
whose integral representations allow the transfer to be constructed
through the underlying neural atoms.

\begin{definition}[Variation space induced by a feature map] 
\label{def:Barron} Let $\mathcal{M} 
(\mathcal{W})$ be the space of finite signed measures over $\mathcal{W}$ and $\phi: \mathcal{W} \times \mathcal{X} \rightarrow \mathbb{R}$, then the variation space is given by
\[
    \mathcal{B}_{\phi} = \{ f: \mathcal{X}\rightarrow \mathbb{R} : f =\int_{\mathcal{W}}\phi(\bm{w}, \cdot) \ d\mu,\  \mu \in \mathcal{M}(\mathcal{W}) \}
\]
with norm,
\[
    \lvert\lvert f \rvert\rvert_\phi = \inf \{\lvert \mu \rvert_{TV}: f = \int_{\mathcal{W}}\phi(\bm{w}, \cdot) \ d\mu,\ \mu \in \mathcal{M}(\mathcal{W}) \}.
\]
\end{definition}
In particular we use the activity features \( \phi_z(\bm w,\bm x)\defeq \sigma(\neuronin) \) and denote the corresponding variation space by \(\mathcal B_z\). 
\footnote{We use the symbol $\mathcal{B}$ due to the link with Barron spaces, but for nonhomogeneous activations it need not coincide with parameter-weighted definitions of Barron space \citep{eBarronSpaceFlowinduced2021, e.RepresentationFormulasPointwise2022}.}

The same first-layer weights determine both the forward activities used
for prediction and the gradient features available to the optimiser.
Our first result relates the intrinsic dimensions of these two
collections.

\begin{proposition}[Forward--backward width transfer]
\label{prop:main-width-transfer}
Under the bounded-weight, activation-regularity, and input-moment
conditions of \cref{ass:bounded-weight,ass:smooth,ass:input-second-moment},
there exists a constant \(C_{\mathrm{tr}}>0\), independent of \(D\),
such that
\[
    d_{D^2n+D}
    \left(
        \mathcal Q(W),
        L_2(p_x)
    \right)
    \leq
    C_{\mathrm{tr}}D^{1/4}
    d_n
    \left(
        \mathcal Z(W),
        \mathcal B_z
    \right)^{1/2}.
\]
\end{proposition}
The proof separates neurons with a sufficiently large incoming weight
from neurons whose incoming weights are uniformly small. For the first
group, Barron approximations transfer through differentiation; for the
second, the tangent features are close to a low-dimensional linear
class. 

The $D^2$ inflation comes from using the approximation directions $\{x_i\partial_{x_j}v \}$. The generic transfer treats all signed tangent directions as potentially distinct. Sparse parity permits a sharper accounting because coordinatewise sign reflections leave the magnitude of every target-dependent query response unchanged. For a weight family \(W\), define its unsigned  representation and backwards families as: 
\[ 
    \mathcal Z^+(W) \defeq \{ \sigma(|\bm w|\cdot\bm x): \bm w\in W \}
    \qquad \mathcal Q^+(W) \defeq \{ x_i\sigma'(|\bm w|\cdot\bm x): \bm w\in W, i \in [D] \}. 
\]

\begin{proposition}[Parity-relevant forward--backward transfer]
\label{prop:parity-width-transfer}
Suppose the ambient feature dictionary satisfies \( \sup_{\bm w\in\mathcal W} \norm{\bm w}_1 \leq C_1\) along with activation-regularity and input-moment conditions \cref{ass:smooth,ass:input-second-moment}. Then there exists a constant
\(C_{\mathrm{tr}}^+>0\), independent of \(D\), such that
\[
    d_{Dn+D}
    \left(
        \mathcal Q^+(W),
        L_2(p_x)
    \right)
    \leq
    C_{\mathrm{tr}}^+D^{1/4}
    d_n
    \left(
        \mathcal Z^+(W),
        \mathcal B_z
    \right)^{1/2}.
\]
\end{proposition}

\subsection{Critical Initial Effective Width}
We combine the forward--backward transfer with a statistical-query-dimension argument. If the initial gradient-query family lies within \(\delta\) of an \(R\)-dimensional subspace, then all available queries are approximated by \(R\) basis directions, with a \(\delta\)-loss in effective tolerance. Insufficient query width then leaves at least two targets compatible with an all-zero transcript. For a zero-respecting learner, this transcript freezes the first-layer
representation indefinitely.

\begin{theorem}[Critical initial effective width]
\label{thm:critical-initial-width}
Let \( d_{\mathrm{SQ}}=\SQdim(\mathcal F,\gamma).\) Under the conditions of \propref{prop:main-width-transfer}, \(\exists C_q > 0\), independent of $D$, such that any zero-respecting learner that achieves squared loss below \(\frac{1-\gamma}{2}\) uniformly over \(\mathcal F\) satisfies
\[
    R_{\mathcal B_z}
    \left(
        \mathcal Z_0,\delta
    \right)
    \geq
    \frac{
        (d_{\mathrm{SQ}}-1)
        \left(
            \tau-C_{\mathrm{tr}}D^{1/4}\sqrt{\delta}
        \right)^2
    }{
        C_qD^2
        \left[
            1+(d_{\mathrm{SQ}}-1)\gamma
        \right]
    }.
\]
for all
 \(
    0\leq\delta
    <
    \frac{\tau^2}{
        C_{\mathrm{tr}}^2D^{1/2}
    }.
\)

\end{theorem}

\begin{corollary}\label{cor:unsigned-critical}
Suppose additionally, the target class satisfies the symmetry condition $\abs{ \ip{q_{\bm w,i}}{h} } = \abs{ \ip{q_{\abs{\bm w},i}}{h} }$ for every $h \in \mathcal F$, $i \in [D]$ and $\bm{w} \in \mathcal{W}$. Then, under the conditions of \propref{prop:parity-width-transfer}: 
\[ 
    R_{\mathcal B_z} \left( \mathcal Z_{0}^{+},\delta \right) 
    \geq \frac{ 
        (d_{\mathrm{SQ}}-1) \left( \tau-C_{\mathrm{tr}}^+D^{1/4}\sqrt{\delta} \right)^2 
        }
        { 
            C_q^+D \left[ 1+(d_{\mathrm{SQ}}-1)\gamma \right] 
        },
\]
for all \(
0 \leq \delta < \frac{ \tau^2 }{ (C_{\mathrm{tr}}^+)^2D^{1/2} }.
\)
\end{corollary}
At any transfer-compatible resolution \( \delta = O\left( \frac{\tau^2}{\sqrt D} \right) \), with a sufficiently small dimension-independent constant, and for nearly orthogonal concept classes, these lower bounds simplify to
\[
        R_{\mathcal B_z}
        \left(
            \mathcal Z_0,\delta
        \right)
        =
        \Omega\left(
            \frac{d_{\mathrm{SQ}}}{D^2}\tau^2
        \right), 
        \quad \text{and} \quad        
        R_{\mathcal B_z}
        \left(
            \mathcal Z_0^{+},\delta
        \right)
        =
        \Omega\left(
            \frac{d_{\mathrm{SQ}}}{D}\tau^2
        \right).
\]
under the respective assumptions of \thmref{thm:critical-initial-width} and  \corref{cor:unsigned-critical}.

\subsection{Sparse Parity Bounds}
The upper bound is based on the Fourier-gap mechanism for ReLU majority neurons \cite{barakHiddenProgressDeep2022},
which allows the discovery of parity at $\tau=O(D^{-(k-1)/2})$. To operate at coarser precision, we replace full-support majority neurons with sparse pooled \(\tanh\) neurons.
For a pool
\(U\subseteq[D]\) of size \(s\), let
\[
    z_U(\bm x)
    \defeq
    \tanh
    \left(
        \frac1{\sqrt s}
        \sum_{i\in U}x_i-b_{k-1}
    \right),
    \quad
    h_U(\bm x)
    \defeq
    \sech^2
    \left(
        \frac1{\sqrt s}
        \sum_{i\in U}x_i-b_{k-1}
    \right).
\]
For fixed \(k\), the relevant Walsh coefficients of $h_U$ satisfy
\[
    \lambda_{k-1,s}
    =
    \Theta_k
    \left(
        s^{-(k-1)/2}
    \right),
\qquad
    \lambda_{k+1,s}
    =
    O_k
    \left(
        s^{-(k+1)/2}
    \right).
\]
The coefficient gap makes a target coordinate \(d\in S\) detectable whenever a pool contains \(S\setminus\{d\}\). Smaller pools strengthen the signal and tolerate coarser feedback, but require more neurons to provide a witness pool for every target coordinate.
\begin{theorem}[One-round sparse-parity upper bound]
\label{thm:one-round-parity-upper}
Suppose \(k\leq s<D\) and
\[
    2\tau
    <
    \min
    \left\{
        \abs{\lambda_{k-1,s}},
        \abs{
            \lambda_{k-1,s}
            -
            \lambda_{k+1,s}
        }
    \right\}.
\]
Then there exists a network of width, 
\[
    P
    =
    O_k
    \left(
        \left(\frac Ds\right)^{k-1}
        \log D
    \right)
\]
along with a zero-respecting update rule whose gradient transcript identifies every \(k\)-sparse parity in one
non-adaptive oracle round.
\end{theorem}
The construction chooses a covering family of size-\(s\) pools such
that, for every target \(S\) and every coordinate \(d\in S\), some pool
contains \(S\setminus\{d\}\). Each target coordinate therefore has a
witness gradient response at the large order-\((k-1)\) coefficient,
while non-target coordinates produce only zero or a lower-order
coefficient. Thresholding across all pools permits robust recovery of
the complete support in one non-adaptive round.
\begin{corollary}[Precision-width interpolation]
\label{cor:precision-diversity}
Let \(k\) be fixed and suppose \(\tau^2=\Theta(D^{-\alpha}),\) where \( 0<\alpha<k-1\). Then there exists a one-round learner of width
\[
        P
        =
        O_k
        \left(
            D^{k-1-\alpha}\log D
        \right).
\]
\end{corollary}
The construction interpolates between two limiting strategies (\cref{sec:endpoints}). At
precision \( \tau^2=O_k(D^{-(k-1)}) \), $O(k)$ neurons suffice, where the $k$ scaling is needed to represent the solution. At constant precision and fixed \(k\),
a static dictionary containing representations of every
\((k-1)\)-parity has width \(O_k(D^{k-1})\) and also learns in one
round.

Since a collection of \(P\) activities has effective width at most \(P\) at every nonnegative resolution, the constructive width bound provides an upper bound on critical effective width. Since the construction uses only positive weights, its width bound also applies to unsigned effective width. Combining \corref{cor:unsigned-critical} with \corref{cor:precision-diversity} therefore characterises the critical unsigned width up to logarithmic factors: 
\[
\boxed{
    R_{\mathrm{critical}}^{+} = \widetilde{\Theta}_k(D^{k-1}\tau^2).
    }
\]

\subsection{Effective Width and Fixed-Feature Fitting}
\paragraph{Generalisation.}
Effective width also controls the statistical complexity of a frozen representation when evaluated on fresh data. If \(\mathcal Z\) has effective width \(R\) at resolution \(\delta\), then every linear readout satisfying \(\norm{\bm a}_1\leq A\) obeys, schematically (\propref{prop:width-generalisation}), \[ \mathcal L(f)-\widehat{\mathcal L}_N(f) \lesssim A\left[ (1+\delta) \sqrt{\frac{\min\{R,C_w^2D^2\}}{N}} +\delta \right] + \sqrt{\frac{\log(1/\eta)}{N}}, \] where \(\lesssim\) suppresses fixed activation- and loss-dependent constants. The two terms in the minimum arise respectively from Rademacher-complexity estimates for the \(R\)-dimensional approximating space \citep{bartlettRademacherGaussianComplexities2001} and for the global variation-norm ball \citep{bachBreakingCurseDimensionality2017}. In particular, although discovering a \(k\)-sparse parity may require effective width polynomial in \(D\), once its support is known the target admits an \(O(k)\)-width representation. For controlled readout norm and approximation error, the resulting fixed-feature complexity bound is independent of the ambient dimension.

\paragraph{Interpolation.}
At initialisation, fixed-feature interpolation imposes a complementary constraint on the readout norm. For sparse parities, suppose that the representation and its \(R\)-dimensional approximating space are chosen independently of the parity target. If \(R/N\to0\), then, for every fixed parity target, its empirical projection onto the target-independent approximating space vanishes with high probability. Any exactly interpolating readout must therefore fit an order-one component through the residual features, whose variation norms are at most \(\delta\), and hence 
\[
    \norm{\bm a}_1 \geq \frac{1-o(1)}{\delta} 
\] 
with probability \(1-o(1)\) over the initialisation and sample (\propref{prop:parity-interpolation-cost}). Thus \(R\asymp N\) marks a crossover between a low-width regime in which interpolation forces norm inflation and one in which the dominant approximating space may itself support interpolation.

\subsection{Predicted Phase Diagram}
To derive broader predictions for the typical training regime of neural networks (stochastic feedback, task-agnostic initialisation and update rule) from the oracle bounds, we make two additional modelling assumptions.

\paragraph{Sample accumulation.} The formal results above concern a single coordinatewise
tolerance-\(\tau\) oracle call. To obtain a trajectory-level
interpretation, we assume that optimisation time can be traded for
statistical precision, capturing possible hidden-progress effects in
SGD. For a fixed query, averaging its stochastic gradient over \(N\)
fresh samples produces fluctuations of order \(N^{-1/2}\), motivating
\(
    \tau^2
    \asymp
    N^{-1}.
\)

This interpretation is exact only when the representation remains fixed
during accumulation. Along an evolving trajectory, successive samples
estimate gradients at different representations, so their information
need not combine as one fixed-query average. For sparse parity, we
nevertheless use this scaling as a trajectory-level approximation,
motivated by the initial slow-learning regime in which the relevant
population correlations are weak and the representation changes only
gradually.\footnote{
Even when sample averages concentrate, fixed-step SGD does not
automatically translate increasing certainty into progressively stable
parameter estimates; learning-rate decay or explicit averaging may be
required.
}

\paragraph{Operational width.} We assume that the relevant discovery scaling is given by the unsigned representation, corresponding to the nonnegative-weight regime in which \(\mathcal Z=\mathcal Z^+\) and the sharper \(D\)-inflation transfer applies. In principle, coordinated sign cancellations could make ordinary effective width substantially smaller than unsigned width. Realising and exploiting such a dense encoding, however, may require finely tuned networks and structured decoding of the resulting gradient responses. We therefore conjecture that unsigned width better captures the operationally relevant regime for typical neural-network training.

\subsubsection{Effective width--sample phase diagram}
Under the fresh-sample and unsigned scaling assumptions, the sparse-parity bounds yield the task frontier
\[
        N R_{\mathrm{task}}(N)
        =
        \widetilde\Theta_k
        \left(
            D^{k-1}
        \right).
\]
A narrow representation consequently requires more samples before the
target becomes visible, whereas greater functional diversity exposes
the target at lower statistical precision.

The task-visibility frontier and the fixed-feature crossover divide the
width--sample plane into qualitatively different regimes. Below the
task frontier, the target is not visible through gradient feedback.
Below the fixed-feature crossover \(R\asymp N\), exact interpolation
requires readout-norm inflation; beyond it, the dominant feature space
may itself support interpolation, so additional norm control may be
needed to select a generalising solution.
Figure~\ref{fig:precision-diversity-frontier} summarises this picture.

\section{Experiments}

Under the previous modelling assumptions, the task frontier motivates a corresponding trajectory-level
width--time hypothesis. Let \(T_\varepsilon\) denote the first time at
which worst-case accuracy reaches \(1-\varepsilon\). If successive
training steps contribute comparable increments of fresh statistical
information, then a trajectory of approximately constant effective
width \(R\) should obey the lower envelope
\(T_\varepsilon R=\widetilde\Omega_k(D^{k-1})\). For time-varying
representations, this suggests replacing \(T_\varepsilon R\) by the
cumulative expenditure
\(\mathcal C_{T_\varepsilon}(\delta)
\defeq\sum_{s\leq T_\varepsilon}
R_{\mathcal B_z}(\mathcal Z^s,\delta)\). The empirical prediction is
therefore that successful trajectories require
\(\mathcal C_{T_\varepsilon}(\delta)
=\widetilde\Omega_k(D^{k-1})\), trading lower instantaneous width
against longer optimisation time.

\subsection{Estimating Effective Width}
Computing Kolmogorov \(n\)-widths in Variation space is generally an infinite-dimensional optimisation problem. Sparse parity provides a tractable setting because the Hadamard--Walsh functions form a finite orthogonal basis of \(L_2(p_x)\). We expand each neuron activity in this basis and construct a low-dimensional approximation to the activity set, weighting each Walsh coefficient by an upper bound on its representation cost in \(\mathcal B_z\). 

We define \(\widehat R_t(\delta)\) as the smallest subspace dimension for which the resulting weighted residual is at most \(\delta\), and define the cumulative effective-width upper bound by 
\[ 
   \widehat{\mathcal C}_t(\delta) \defeq \sum_{s=1}^{t} \widehat R_s(\delta). 
\] 
The construction is guaranteed to upper-bound effective width, although
it need not be tight because it searches over a restricted family of
Walsh-based approximations and the representation cost is not necessarily minimal. 
Consequently, low estimated width certifies a compact functional approximation, whereas high estimated width may partly reflect looseness of the Walsh restriction.

\subsection{Cumulative Effective Width and Ambient Dimension}
We focus on \(k=3\), for which the trajectory-level prediction is 
\[ 
    \widehat{\mathcal C}_{T_\varepsilon}(\delta) = \widetilde\Omega(D^2),
\] 
and consider ambient dimensions \(D\in\{5,7,9,11\}\). For each hyperparameter setting, indicated by colour, we evaluate
multiple targets, initialisations, and feedback realisations. Thin
curves show individual runs, while the thick curve reports the minimum
test accuracy across the evaluated runs, providing an empirical analogue
of the theory's uniform target-and-feedback criterion.

Figure~\ref{fig:vary-k-ib} plots accuracy against 
effective-width, in the style of information-plane
complexity--accuracy plots. The instantaneous effective-width
trajectories depend strongly on the initialisation scale, and successful
runs do not always exhibit a pronounced transient increase. This is
consistent with the width--time interpretation: trajectories that
maintain low instantaneous effective width may compensate through a
longer time to reach the target.

By contrast, cumulative effective-width expenditure exhibits a clearer
relationship with task accuracy and is less sensitive to
initialisation; see \cref{fig:vary-k-loglog}. Finally, normalising the
cumulative effective-width upper bound by \(D^2\) produces an
approximate collapse over the dimensions tested
(\cref{fig:vary-k-collapse}), consistent with the predicted
\(\widetilde\Omega(D^2)\) dependence.
\begin{figure*}[t]
    \centering
    \begin{subfigure}[t]{0.85\textwidth}
        \includegraphics[width=0.95\linewidth]{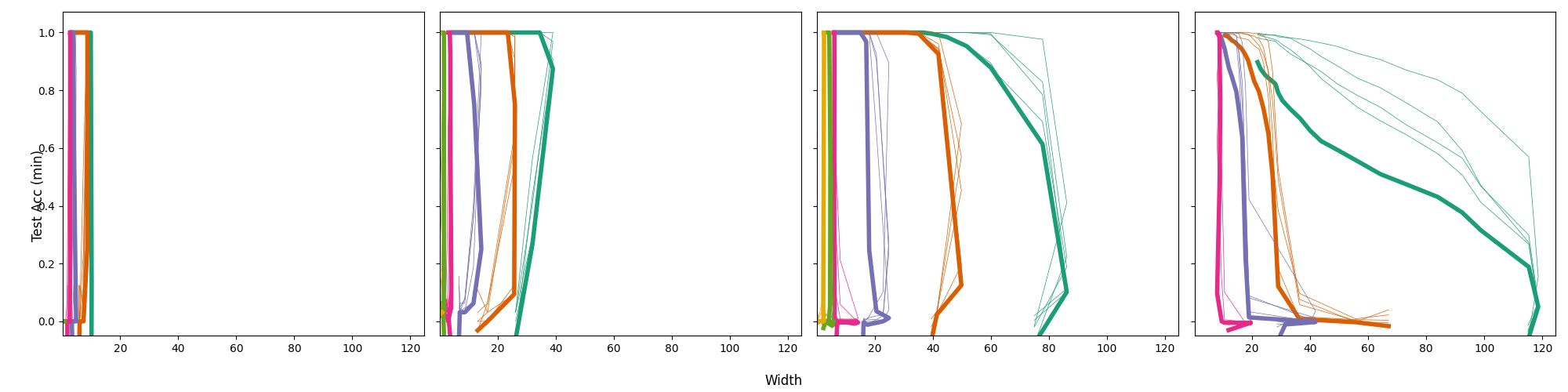}
        \caption{Instantaneous representation complexity.}
        \label{fig:vary-k-ib}
    \end{subfigure}
    \centering
    \begin{subfigure}[t]{0.4\linewidth}
        \centering
        \includegraphics[height=1.5in]{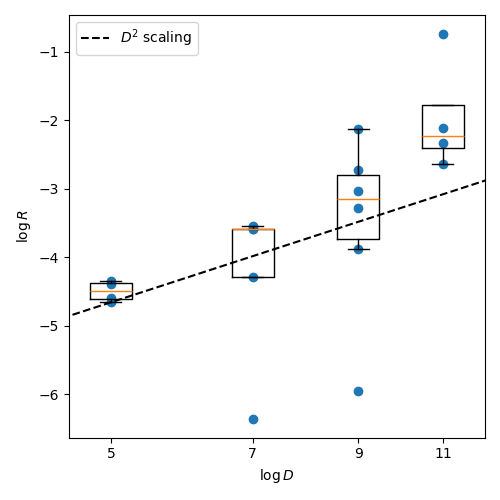}
        \caption{Scaling with ambient dimension.}
        \label{fig:vary-k-loglog}
    \end{subfigure}%
    ~
    \begin{subfigure}[t]{0.4\linewidth}
        \centering
        \includegraphics[height=1.65in]{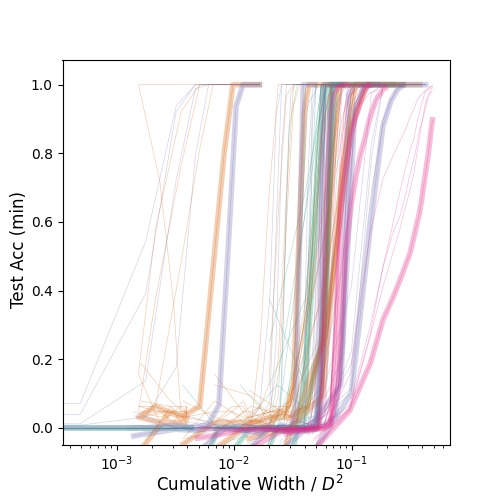}
        \caption{Scaling normalised by $D^2$.}
        \label{fig:vary-k-collapse}
    \end{subfigure}

    \caption{
    \small Representation complexity during learning and its scaling with problem dimension.
    Results are shown for $k=3$ sparse parity across ambient dimensions
    $D\in\{5,7,9,11\}$, with colours denoting different initialisation scales.
    (a) Instantaneous width $\widehat{R}_t$ can vary substantially across
    successful trajectories, with small initialisations exhibiting little to no
    transient increase. (b) The distribution of $\widehat{\mathcal C}_{\varepsilon}$ across hyperparameter settings as a function of $D$, against the predicted $D^2$ scaling. (c) Normalising $\widehat{\mathcal C}_{\varepsilon} / D^2$ produces an approximate collapse across the tested dimensions.
    }
    \label{fig:vary-k}
\end{figure*}

\subsection{Effective Width, Optimisation Time, and Generalisation}  

The trajectory-level prediction also implies a width--time constraint.
For a prescribed accuracy \(1-\varepsilon\), we consider the training and test stopping times $T^{\rm{train}}_{\epsilon}$ and $T^{\rm{test}}_{\epsilon}$ and the corresponding average empirical effective-width upper bound before these
stopping times
\[
    \overline R_{\varepsilon}^{\mathrm{x}}(\delta)
    \defeq
    \frac{1}{
        T_{\varepsilon}^{\mathrm{x}}
    } \widehat{\mathcal C}_{
        T_{\varepsilon}^{\mathrm{x}}
    }(\delta) \quad \mathrm{x} \in \{\rm{train}, \rm{test}\}. 
\]
Under approximately uniform information accumulation across training
steps, the cumulative-width prediction gives the lower-envelope
constraint
\[
    T_{\varepsilon}^{\mathrm{x}}
    \overline R_{\varepsilon}^{\mathrm{x}}(\delta)
    =
    \widetilde\Omega_k(D^{k-1}), \quad \rm{x} \in \{\rm{train}, \rm{test}\}.
\]
Thus, trajectories cannot simultaneously maintain low average
effective width and reach either performance threshold in few
optimisation steps.

For \(k=3\), \(D=10\) parity, the training-time frontier remains close
to the inverse lower-envelope scaling
\(
    \overline R_\varepsilon^{\mathrm{train}}
    \propto 1 / T_\varepsilon^{\mathrm{train}}
\) (\cref{subfig:log-log,subfig:time-train-acc}). The test-time frontier obeys the same one-sided constraint but changes
character beyond a characteristic effective width. Increasing width
continues to accelerate fitting without producing a corresponding
reduction in
\(T_\varepsilon^{\mathrm{test}}\), causing the test-time lower envelope
to become approximately U-shaped (\cref{fig:u-shape}). 

The emergence of a characteristic width in the test-time frontier is
qualitatively consistent with the sample-supported crossover
\(R\asymp N\): beyond this scale, greater effective width can continue
to accelerate fitting while no longer improving the low-dimensional population certificate. We hypothesise that rapid interpolation by the second layer may suppress the residual error signal driving first-layer adaptation, separating fast empirical fitting from slower feature discovery.


\begin{figure*}[t!]
    \centering
    \begin{subfigure}[t]{0.3\textwidth} 
        \centering
        \includegraphics[height=1.5in]{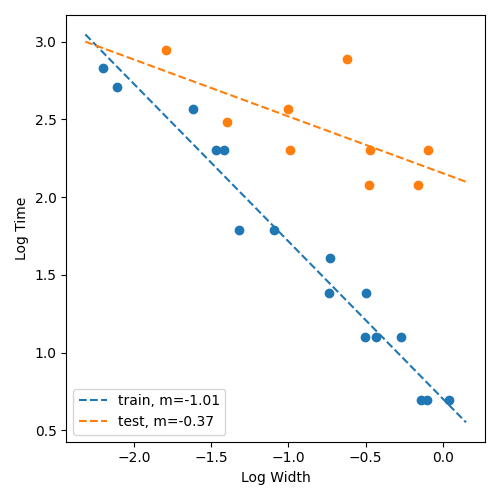}
        \caption{Accuracy against cumulative complexity, log-log plot.}
        \label{subfig:log-log}
    \end{subfigure}%
    ~
    \begin{subfigure}[t]{0.3\textwidth}
        \centering
        \includegraphics[height=1.55in]{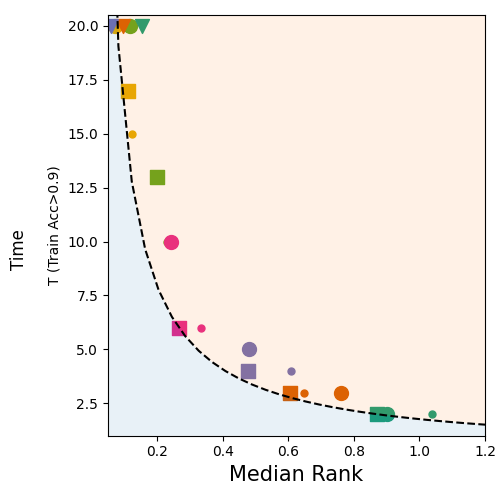}
        \caption{Time (train) against complexity.}
        \label{subfig:time-train-acc}
    \end{subfigure}%
    ~
    \begin{subfigure}[t]{0.3\textwidth}
        \centering
        \includegraphics[height=1.55in]{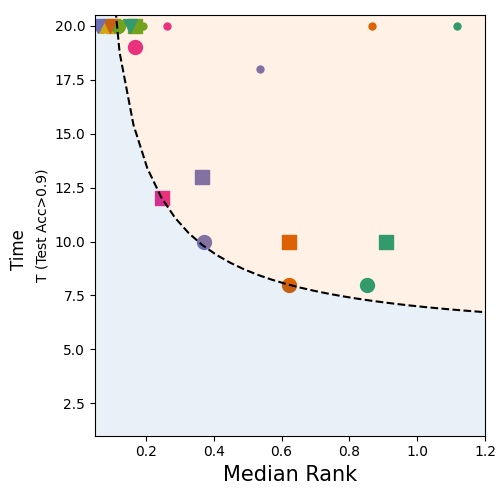}
        \caption{Time (test) against complexity.}\label{fig:u-shape}  
    \end{subfigure} 
    \caption{\small
    Width--time frontiers for \(k=3\), \(D=10\) parity across
    initialisation scales (colour) and weight-decay values (marker).
    (a) Accuracy plotted against cumulative effective-width expenditure.
    (b) The training-time frontier remains close to the predicted inverse
    lower envelope and is comparatively insensitive to weight decay.
    (c) The test-time frontier departs upward beyond a characteristic
    effective width, resulting in a U-shaped lower envelope. Runs with
    stronger weight decay remain closest to the frontier, while the growing
    gap
    \(T_{0.1}^{\mathrm{test}}-T_{0.1}^{\mathrm{train}}\)
    indicates delayed generalisation.
    }    
\end{figure*}

\section{Discussion}

Our results separate the resources required to represent a solution from those required to discover it. A \(k\)-sparse parity has a compact representation once its support is known, but noisy gradient feedback requires polynomially greater effective width to identify that support. This gives overparameterisation a computational interpretation: useful width does not merely enlarge the predictor class, but creates distinct functional directions through which the unknown target can influence the optimiser. The relevant resource is therefore effective width, not the raw number of neurons.

\subsection{Limitations and Future Directions}

The largest unresolved gap is the factor \(D\) between the general signed-weight transfer and its unsigned equivalent. Given an \(n\)-dimensional forward approximating space \(V\), the generic proof must span the backward features \(\{x_i\partial_{x_j}V:i,j\in[D]\}\), leading to a \(D^2\) inflation. For parity, sign invariance of the query magnitudes reduces this to \(D\) after symmetrisation, but it remains unclear whether the additional factor can be realised by a learnable signed representation or is only an artefact of the task-agnostic argument. 

Resolving this question may require modelling the restrictions that neural-network training places on how gradient responses are combined and decoded. Effective width provides a common functional geometry for both the information available during target discovery and the forward complexity governing interpolation and population control. A natural next step is to make
this geometry more sensitive to the information operationally accessible
to the optimiser without losing that shared control.

\paragraph{Task specificity.} The width-transfer and SQ lower-bound framework applies to a broad class of smooth activations and concept classes, but our analysis is currently limited to shallow networks. Extending it to deep networks would require tracking how representations and target-dependent information are transformed and shared across layers. The presented upper bounds are more problem-specific: they exploit the permutation symmetry and explicit Fourier--Walsh structure of sparse parity. Future work could study effective width in other problems that combine compact target representations with difficult feature discovery, including Gaussian single-index models \cite{arousOnlineStochasticGradient2021a, damianSmoothingLandscapeBoosts2023} and other concept classes with large SQ dimension \citep{feldmanGeneralCharacterizationStatistical2017}.

\paragraph{Optimisation dynamics.}
Rather than attempting to solve the full finite-width SGD dynamics, we
use a coordinatewise adversarial oracle to derive information
constraints that apply across a broad class of update rules. We
interpret its tolerance through fresh-sample accumulation,
\(\tau^2\asymp N^{-1}\). In line with hidden-progress analyses
\citep{barakHiddenProgressDeep2022}, this interpretation assumes that
weak population correlations can be resolved by accumulating statistical
information across samples. Whether sample-wise gradient methods are
fundamentally governed by this SQ-style scaling, or can exploit
information in individual samples to attain more favourable PAC-type
scaling
\citep{abbeUniversalityDeepLearning2020,barakHiddenProgressDeep2022,damianSmoothingLandscapeBoosts2023},
remains an open question.

At the same time, the broad class of zero-respecting update rules leaves open how efficiently population correlations are used. At low effective width, constant-learning-rate SGD may use weak correlations less efficiently than explicit averaging. Consistent with this possibility, standard Gaussian initialisation and Langevin-like perturbed gradient dynamics can remain uninformative for fixed parity targets when the relevant Fourier correlations are small \citep{shoshaniHardnessLearningFixed2025}. Annealing, iterate averaging, or increasing batch size may be required to approach the frontier. At high effective width, rapid second-layer interpolation may instead weaken the residual signal driving first-layer feature learning. Characterising these nonlinear departures for more tightly specified optimiser classes is an important direction for future work.

\paragraph{Measuring effective width.} Effective width is resolution-dependent and generally defined through
an infinite-dimensional functional optimisation. For sparse parity, the
finite input space and explicit Fourier--Walsh basis yield a
constructive, although potentially loose, upper bound. This strategy
may extend to other tasks with known input distributions and appropriate
orthogonal decompositions, but estimating the effective width on an unknown
distribution remains open. The essential property of the variation geometry is not the particular use
of total variation itself, but that its coefficient norm controls neural
representation cost while preserving the atomwise structure needed for
the forward--backward transfer. Future work could seek alternative
coefficient norms or feature-space relaxations that preserve this
transfer while making effective width easier to estimate from data. 

\paragraph{Accessible information.} Unlike the noise-regularised mutual information \(I[X;Z+\epsilon]\) used in analysis of deterministic networks \cite{shwartz-zivOpeningBlackBox2017, saxeInformationBottleneckTheory2019}, whose noise distribution models finite numerical precision, effective width uses a scalar resolution \(\delta\) to quantify the functional directions linearly accessible to prediction and gradient feedback. This gives resolution a direct operational role in learning and can simplify its estimation in structured settings. A measure adapted to stochastic gradient feedback may provide an intermediate notion, retaining effective width’s operational focus on linearly accessible directions while measuring their distinguishability under sampling noise rather than worst-case tolerance. 

\subsection{Conclusion}
More broadly, effective width exposes two sides of
overparameterisation. Greater functional diversity can weaken
statistical control, but it can also reduce the precision or time
required for target discovery. The resulting frontier is particularly
relevant when accurate gradient information is expensive, and suggests
a basis for studying compute-optimal learning through intrinsic
trade-offs among effective width, feedback precision, and optimisation
time. Crucially, the resources needed for learning need not resemble
those needed for inference: a predictor that is compact once learned
may require a substantially richer representation to discover.

\newpage
\bibliography{references}

\newpage

\appendix

\section{Preliminaries}

\subsection{Model} \label{app:model}
We recall our model $f(\ \cdot\ ; \theta): \mathbb{R}^D \rightarrow\mathbb{R}$,
\[
f(\bm{x}; \theta)
    =
    \sum_{p=1}^{P} a_{p} \sigma(\bm{w}_{p}\cdot \bm{x}).
\]
Where $\theta =(\bm{a}, W)$, with $W = [\bm{w}_1,\cdots, \bm{w}_P] \in \mathbb{R}^{P\times D}$ and $\bm{a} = (a_1,..,a_P) \in \mathbb{R}^P$. Where we assume $\bm{w}_i \in \mathcal{W}$ is a compact subset of $\mathbb{R}^D$.

The optimiser does not observe the target distribution directly; instead, it receives noisy gradient information through an oracle,  
\[\hat{g}_{p,d}^t = g_{p,d}^t+\xi_{p,d}^t, \quad \xi_{p, d}^t \in [-\tau, \tau].\]
Where $ g_{p,d}^t = \mathbb{E}[x_d\sigma'(\bm{w_p}^t \cdot \bm{x}) y]$ is the population gradient and $\tau$ defines the noise scale of the oracle.

We further restrict the learner to update the weights (and queries) based only on the previous round of feedback through,
\[
    W^{t+1} = W^t + \phi(W^t, \hat{G}^t)
\]
where $\hat{G}^t \in \mathbb{R}^{P\times D}$ is the feedback received, $(\hat{G}^t)_{p,d}=\hat{g}_{p,d}^t$. Where,
\[ \phi: \mathbb R^{P\times D} \times \mathbb R^{P\times D} \rightarrow \mathbb R^{P\times D}, \] is independent of the unknown target except through its arguments. We assume the update rule is zero-respecting in its feedback argument,
\[
    \phi(W, 0_{P\times D}) = 0_{P\times D}. 
\]

\paragraph{Second-layer adaptation and optimiser state.} A fuller model could augment \(W^t\) with second-layer coefficients and other optimiser state. This introduces no uncontrolled target-information channel: second-layer target information is mediated directly by correlations with the current activities, \(\mathbb E[z_p(\bm x)y]\), and is therefore controlled by the effective width of \(\mathcal Z^t\). First-layer gradients instead involve the tangent features \(x_i\sigma'(\bm w_p^t\cdot\bm x)\), necessitating the forward--backward transfer. The second layer and optimiser state may combine the correlations exposed by these families, but cannot create target information absent from them. Such extensions could be represented by time-dependent state updates \[ W^{t+1} = W^t+\phi_t(W^t,S^t,\widehat G^t), \qquad S^{t+1} = \psi_t(W^t,S^t,\widehat G^t). \] They do not affect indistinguishability: under a common all-zero transcript, every target-independent state variable evolves identically for both targets. If \(\phi_t(W,S,0)=0\), the first-layer query family also remains fixed, allowing the zero transcript to continue indefinitely. We omit these additional variables to keep the notation focused on whether the representation exposes detectable target information.

\paragraph{Assumptions}
\begin{assumption}[Bounded weights] \label{ass:bounded-weight}
We assume weights are drawn from a compact subset
$\mathcal{W}\subset \mathbb{R}^D$, with
$\mathcal{W} \subseteq B_\infty(0,C_w)$.
\end{assumption}

\begin{assumption}[Smooth activations] \label{ass:smooth}
We assume
\(
\abs{\sigma(t)} \leq A_{\sigma}(1+\abs{t})
\), 
\(
\abs{\sigma'(t)}\leq B_\sigma
\)
and
\( \norm{\sigma''}_{\infty} \leq C_{\sigma}\).
\end{assumption}


\begin{assumption}[Input second-moment control]
\label{ass:input-second-moment}
There exists a constants \(C_x, \kappa_x<\infty\), independent of \(D\),
such that,
\[ \max_{i\in[D]} \mathbb E[x_i^2] \leq C_x^2. \] and 
\[
    \sup_{\substack{i\in[D]\\ \norm{\bm u}_\infty\leq1}}
    \mathbb E\left[
        x_i^2(\bm u\cdot\bm x)^2
    \right]
    \leq
    \kappa_x^2D.
\]
\end{assumption}

\begin{assumption}[Weights bounded in $l_1$] \label{ass:boundedl1-weight}
$\mathcal{W} \subseteq B_1(0,C_1)$.
\end{assumption}


\section{Width Transfer}

\subsection{Spaces}
We will work with $L_2(p_x)$ and associated inner product. As well as the Barron spaces $B_{\phi}$ with associated norm. Where $\mathcal{M}(\mathcal{W}) = \{\mu: \mu \text{ a signed measure over $\mathcal{W}$}, \lvert \mu \rvert_{TV} < \infty \}$. In particular we will use feature maps $\phi_z(\bm{w}, \bm{x}) = \sigma(\bm{w}\cdot \bm{x})$ and $\phi_{q^i}(\bm{w}, \bm{x}) = x_i\sigma'(\bm{w}\cdot \bm{x})$, corresponding generating the collections $\mathcal{Z}=\{\sigma(\bm{w_p}^t \cdot \bm{x}): p \in [P], t \in [T] \}$ and $\mathcal{Q}^i=\{x_i\sigma'(\bm{w_p}^t \cdot \bm{x}):  p\in[P], t \in [T] \}$.

We will also deal with all the tangent directions at once \( Q = \bigcup_{i \in [D]} Q^i\). Since this collection is not directly introduced by a feature map, we require:

\begin{definition}[Barron Space induced by a collection of feature maps.] 
Let \(\{\phi_i\}_{i=1}^k\) be a collection of feature maps and $B_{\phi_i}$ the corresponding normed spaces. Define:
\[
    \mathcal{B}_{\{\phi_i\}_{i=1}^k} \defeq \text{span}\left( \bigcup_{i} \mathcal{B}_{\phi_i} \right) = \left\{\sum f_i: f_i \in \mathcal{B}_{\phi_i}\right\}
\]
with norm $\norm{f} = \inf \{\sum_i \abs{\mu_i}_{TV} :  f = \sum_{i} A_{\phi_i}\mu_i \}$ or equivalently, \(\norm{f} = \inf \{\sum_i \norm{f_i}_{\phi_i} :  f = \sum_{i} f_i, f_i \in \mathcal{B}_{\phi_i}\} \). 
\end{definition}

\paragraph{Synthesis Maps}. We highlight that these Barron norms can be seen as quotient norms induced by synthesis maps. Let $A_{\phi}: \mathcal{M}(\mathcal{W}) \rightarrow L_2(p_x)$
\[
    (A_{\phi}\mu )(x) = \int_{\mathcal{W}} \phi(\bm{w}, \bm{x}) \ d\mu,
\]
then for $f \in \text{Im}\ \mathcal{M}(\mathcal{W})$, 
\[
    \lvert\lvert f \rvert \rvert_{\phi} = \inf_{\mu: A_{\phi}\mu = f}\lvert \mu \rvert_{TV}. 
\]
for brevity we use the notation $A_z \defeq A_{\phi_z}$ and $A_q^i \defeq A_{\phi_{q_i}}$, and similarly $\mathcal{B}_z$ and $\mathcal{B}_q^i$ for the corresponding Barron spaces.

We further note the combined norm is the quotient norm of the map $S$, operating on the direct sum with \(\norm{\cdot}_1\):
\[
    S: \bigoplus \mathcal{B}_q^i \rightarrow \mathcal{F} 
\]
where $S(f_1 \oplus\dots \oplus f_D) = \sum_{i=1}^d f_i$.

We will also need to be able to differentiate under the integral sign:
\begin{lemma} \label{lemma:dom-convergence}
    Let $f = A_z\mu$. Then
    \[
        \partial_{x_i} f(\bm x)
        =
        \int_{\mathcal W}
        \partial_{x_i}\sigma(\neuronin)
        \,d\mu(\bm w).
    \]
\end{lemma}

\begin{proof}
    This follows from compactness of $\mathcal{W}$ and the linear growth condition. Consider,
    \[
        \abs{\partial_{x_i}\sigma(\neuronin)} = \abs{w_i \sigma'(\neuronin)} \leq \abs{w_i} \abs{1 + A \abs{\neuronin}} \defeq g_{\bm{x}}(\bm{w})
    \]
    Now for any $\bm{x_0}$ and a small closed ball $U$ around it, since $U \times \mathcal{W}$ is compact,
    \[
    g_{\bm{x}}(\bm{w}) \leq \sup_{\bm{x} \in U, \bm{w} \in \mathcal{W}
    } g_{\bm{x}}(\bm{w}) = \bar{C} < \infty.
    \] 
    Hence the integrand is locally uniformly dominated by $\bar{C} \in L_1(\mu)$, and standard differentiation under the integral applies.
\end{proof}

\subsection{Width}

\begin{quote}
\textbf{Definition~\ref{def:k-width} (Kolmogorov $n$-Width).}
Given normed space $X$ and $W\subset X$, the Kolmogorov $n$-width is
\[
    d_n(W, X)
    = \inf_{V\subset X:\dim(V)\leq n}
      \sup_{w \in W}
      \inf_{v \in V} \lVert v-w\rVert_X.
\]
\end{quote}

\subsection{Preliminary Lemmas}

We recall a standard lemma:
\begin{lemma} (Width contraction under bounded maps)
Let $T:X \to Y$ be a bounded linear operator between Banach spaces, and let $W \subset X$. Then
\[
d_n(T(W), Y)
\le
\|T\|_{\mathrm{op}} \, d_n(W, X).
\]
\end{lemma}

We now show that the canonical inclusion $\mathcal{B}_q^i \hookrightarrow L_2(p_x)$ is bounded.

\begin{lemma}
    \( A_q^i\) acts continuously, \(\exists C_{\text{in}} \leq \infty\),
    \[  
        \norm{A_q^i \mu} \leq C_{\text{in}} \abs{\mu}_{TV}.   
    \]
    As a consequence the canonical inclusion,
    \begin{align*}
        I: \ &\mathcal{B}_q^i \rightarrow L_2(p_x) \\
        & [\mu] \mapsto A_q^i\mu,
    \end{align*}
    has $\norm{I}_{op}\leq C_{\text{in}}$.
\end{lemma}

\begin{proof}
    \[
        \abs{A_q^i \mu} 
        = \abs{x_i\int \sigma'(\neuronin) \ d \mu}
        \leq
        \abs{\mu}_{TV}\abs{x_i}\abs{\sup_{\bm{w}\in \mathcal{W}} \sigma'(\neuronin)}.  
    \]
    Using,
    \[ 
      \abs{\sup_{\bm{w}\in \mathcal{W}} \sigma'(\neuronin)} 
      \leq B_{\sigma},  
    \]
    So we have, 
    \[
        C_{\text{in}} \leq  B_{\sigma}C_x.
    \]
\end{proof}

\begin{corollary}
        \(
            d_n(\mathcal{Q}^i, L_2) \leq C_{\text{in}}\cdot d_n(\mathcal{Q}^i, \mathcal{B}_q^i).
        \)
\end{corollary}

We now proceed, as suggested in the main text, by breaking the neurons into big and small sets.

First we introduce a couple of auxiliary lemmas enabling us to switch ambient spaces.

\begin{lemma}[Combining coordinate queries] 
\label{lemma:union-bound}
    \( d_{D\cdot n}(\mathcal{Q}, \mathcal{B}_q) \leq \max_{i} d_{n}(\mathcal{Q}^i, \mathcal{B}_q^i) \).
\end{lemma}
\begin{proof}
    Let $(V_i, \epsilon_i)$ be the approximating subspace and their errors, and consider the set:
    \[
        V = \text{span}\left\{\bigcup_{i}^D \{e_{1:n}^i:  \langle e_{1:n}^i \rangle = V_i\} \right\}.
    \]
Then we have $\text{rank}(V)\leq D\cdot n$. 
Now consider $f\in Q$, $f$ belongs to at-least one $Q^i$ so we have $v \in V_i$ such that:
\[
    \norm{f - v}_q = \inf\{\sum \norm{f_i}_{q^i}: \sum_i f_i = f - v \} \leq \norm{f - v}_{q^i} \leq \epsilon_i
\]
Therefore the worst case approximation is bounded:
\[
    d_{D\cdot n}(\mathcal{Q}, \mathcal{B}_q) \leq \sup_{f \in \mathcal{Q}}\inf_{v\in V}\norm{f - v}_{q} \leq \max_i {\epsilon_i}. 
\]
\end{proof}

\begin{lemma} \label{lemma:to-tangent}
    Consider the feature map $t(\bm{w}, \bm{x}) = \sigma'(\neuronin)$ and the associated empirical features $\mathcal{T}=\{\sigma'(\bm{w}_p^t\cdot \bm{x}): p \in [P], t \in [T]\}$. We have:
    \[
        d_n(\mathcal{Q}^i, \mathcal{B}_q^i) \leq d_n(\mathcal{T}, \mathcal{B}_{t}) 
    \]
\end{lemma}
\begin{proof}
    Let \( (V, \epsilon) \) be the approximating subspace and error for $\mathcal{T}$, we set $V_{q} \defeq \{x_i v: v \in V\}$ then for any $q^i \in Q^i$ we have $v\in V$ such that:
    \[
        \norm{v - \frac{q^i}{x_i}}_{t} \leq \epsilon 
        \implies \exists \mu \ \text{s.t.}\  A_{t}\mu = v - \frac{q^i}{x_i} \ \text{and} \ \norm{\mu}_{TV}\leq \epsilon
    \]
    Given $x_i A_t = A_q^i$, we have:
    \[
        \exists \mu \ \text{s.t.}\  A_q^i\mu = x_iv - q^i \ \text{and} \ \norm{\mu}_{TV}\leq \epsilon \implies \norm{x_i v - q}_{q^i} \leq \epsilon.
    \]
    
\end{proof}

\begin{lemma} \label{lemma:commuting}
    Let $\mathcal{W}_{-\alpha} = \mathcal{W}\setminus B_0(\alpha)$
    The following diagram commutes (as sets):
    \[
        \begin{tikzcd}
        \mathcal{M}(\mathcal{W}_{-\alpha}) \arrow[r, "A_z"] \arrow[d, "M_i"']
          & L_2(p_x) \arrow[d, "\partial_{x_i}"] \\
        \mathcal{M}(\mathcal{W}_{-\alpha})  \arrow[r, "A_t"']
          & L_2(p_x)
        \end{tikzcd}
    \]
    where:
    \[
        M_i: \mu \mapsto \nu \ \text{s.t.} \ \frac{d \nu}{d\mu} = w_i
    \]
\end{lemma}
\begin{proof}
    For \(\mu \in \mathcal{M}(\mathcal{W}_{-\alpha})\) define $M_i\mu$ by:
    \[
        (M_i\mu)(A) \defeq \int_{A} w_i \ d \mu
    \]
    Note by \cref{lemma:dom-convergence}, 
    \[
    \partial_{x_i}(A_z \mu) = \int \sigma'(\neuronin) \ w_i \ d \mu = \int \sigma'(\neuronin) \ d (M_i\mu) = (A_{t}M_i)\mu.
    \]
\end{proof}

\begin{lemma}[Width transfer for large neurons.] \label{lemma:large}
    Suppose each set of incoming weights $\bm{w}_p$ has some large coordinate, $\lvert \bm{w}_{p,k(p)} \rvert \geq \alpha$ then:
    \[
        d_{D^2\cdot n}(\mathcal{Q}, \mathcal{B}_q) \leq \frac{C_w}{\alpha} \cdot d_n(\mathcal{Z}, \mathcal{B}_z)
    \]
\end{lemma}

\begin{proof}
    We will first bound each $d_n(\mathcal{T}, \mathcal{B}_t)$ and use \cref{lemma:union-bound} and \cref{lemma:to-tangent} to get the stated bound.

    Let $(V, \epsilon)$ be the approximating space and error for $\mathcal{Z}$. We will show that
    \[
        V' = \operatorname{span}\{ \partial_{x_1}V,\dots, \partial_{x_d}V\}
    \]
    is an approximating subspace for $\mathcal T$.
    
    Then for $t_0 \in \mathcal{T}$ and $z_{0}\in\mathcal{Z}$ generated by the empirical neuron measure $\delta_{\bm{w}_0}$, firstly from the approximating subspace:
    \[
        A_z\delta_{\bm{w}_0} = A_z\mu_{v} + A_z\mu_{e}, \quad \abs{\mu_e}_{TV} \leq \epsilon.
    \]
    Then by \cref{lemma:commuting}, for any $i$ we have:
    \[
        \bm{w}_{0,i} t_0= \bm{w}_{0,i} A_t \delta_{\bm{w}_0}  = \partial_{x_i}A_z \delta_{\bm{w}_0} = \partial_{x_i}A_z \mu_v + \partial_{x_i}A_z \mu_e, 
    \]
    therefore,
    \[
        t_0 = \frac{1}{\bm{w}_{0,i}}A_t M_i \mu_v + \frac{1}{\bm{w}_{0,i}}A_t M_i \mu_e
    \]
    with:
    \[
        \abs{\frac{1}{{\bm{w}_{0,i}}}M_i \mu_e}_{TV} \leq \frac{C_w}{\abs{\bm{w}_{0,i}}}\epsilon.
    \]
    Then choose $i = k(0)$, to produce the error bound $\frac{C_w}{\alpha}\epsilon$. Since $\dim V' \leq D \cdot \dim V$, we have $d_{D\cdot n}(\mathcal{T}, \mathcal B_t) \leq \frac{C_w}{\alpha}d_{n}(\mathcal Z, \mathcal B_z)$. 
    
    Finally, applying \cref{lemma:union-bound} and \cref{lemma:to-tangent},
    \[
        d_{D^2n}(\mathcal{Q}, \mathcal{B}_q) \leq \max_i d_{Dn}(\mathcal{Q}^i, \mathcal{B}_q^i) \leq d_{Dn}(\mathcal{T}, \mathcal{B}_t) \leq \frac{C_w}{\alpha}d_n(\mathcal{Z}, \mathcal{B}_z).
    \]
\end{proof}

\begin{lemma}[Width transfer for large unsigned neurons.] \label{lemma:unsigned-transfer}
    Suppose each set of incoming weights $\bm{w}_p$ has some large coordinate, $\lvert \bm{w}_{p,k(p)} \rvert \geq \alpha$. Additionally suppose $w_{p,j}\geq 0$ for all $j$ and,  
    \[
        \sup_{\bm{w} \in \mathcal{W}}\norm{\bm{w}}_1 \leq C_1
    \]
    Then:
    \[
        d_{D\cdot n}(\mathcal{Q}, \mathcal{B}_q) \leq \frac{C_1}{\alpha} \cdot d_n(\mathcal{Z}, \mathcal{B}_z).
    \]
\end{lemma}
\begin{proof}
    Reinstating the objects of the previous proof. We consider instead a directional derivative and summed multiplier:
    \[
        \Delta \defeq \sum_{i=1}^D \partial_{x_i}, \qquad (M_{\Sigma}\mu)(A) = \int_A\left(\sum_{i=1}^D{w_i}\right) \ d\mu.
    \]
    With the inherited commutation relation $\Delta A_z \mu = A_t M_{\Sigma}\mu$. Then consider the approximating subspace, \( V' = \Delta V\). Let $s_0 \defeq \sum w_{0, i} \geq \alpha$, we have:
    \[
        t_0  = \frac{1}{s_0}\Delta A_z \mu_v + \frac{1}{s_0}\Delta A_z \mu_e
    \]
    and the estimate:
    \[
        \abs{\frac{1}{s_0}M_\Sigma \mu_e}_{TV} \leq \frac{C_1}{\alpha}\epsilon.
    \]
    so by the same logic as the previous proof except with $\dim V' \leq \dim V$ we have:
    \[
        d_{Dn}(\mathcal{Q}, \mathcal{B}_q) \leq \max_i d_{n}(\mathcal{Q}^i, \mathcal{B}_q^i) \leq d_{n}(\mathcal{T}, \mathcal{B}_t) \leq \frac{C_1}{\alpha}d_n(\mathcal{Z}, \mathcal{B}_z).
    \]
\end{proof}

\begin{lemma}[Width of small neurons] \label{lemma:small} 
Suppose each incoming weight has $\max_{i}\abs{\bm{w}_{p,i}} \leq \alpha$ then:
\[
    d_D(\mathcal{Q}, L_2) \leq C_s D^{1/2} \alpha
\]
where $C_{s} \leq \infty$ is dimension independent.
\end{lemma}

\begin{proof}
    Write,
    \(\bm w=\alpha\bm u \) so that \(\norm{\bm u}_\infty\leq1. \)
    By the MVT we have:
    \[
        \abs{\sigma'(\neuronin)-\sigma'(0)} 
        \leq \norm{\sigma''}_{\infty}\abs{\neuronin} 
        \leq \alpha\norm{\sigma''}_{\infty}\abs{\bm{u}\cdot\bm{x}}
    \]
    Therefore:
    \[
        \norm{x_i\sigma'(\neuronin) - x_i \sigma'(0)}_{L_2(p_x)} \leq  \alpha C_{\sigma}\kappa_x D^{1/2}
    \]
    Therefore setting $C_{s} \defeq C_{\sigma}\kappa_x$ taking $V = \text{span}\{x_i\sigma'(0): i \in [D] \}$ gives:
    \[
        d_D(\mathcal{Q}, L_2) \leq C_s D^{1/2}\alpha
    \]
\end{proof}

Finally we give a lemma for combining the two bounds:
\begin{lemma} \label{lemma:combine}
Suppose $Q = Q^{-}\cup Q^{+}$
    \[
        d_{n+m}(Q, X) \leq \max{\{d_n(Q^+,X), d_m(Q^{-},X) \}}.
    \]
\end{lemma}
\begin{proof}
    Let $V^{+}, \epsilon^{+}$ and $V^{-}, \epsilon^{-}$ be approximating subspaces and errors for $Q^{+}$ and $Q^{-}$ respectively. Then take $V\defeq\operatorname{span}\{V^{+}, V^{-}\}$ so $\dim V \leq n + m$.

    For any $q \in Q$ then either $q \in Q^{-}$ or $q\in Q^{+}$ therefore:
    \[
        \text{dist}(q, V) \leq \max{\{\epsilon^{-}, \epsilon^{+}\}}.
    \]
\end{proof}

\subsection{Main transfer result}
Combining the small and large lemmas gives us:
\begin{proposition}[Width transfer] \label{prop:transfer} there exists a dimension-free constant $C_{tr} < \infty$ such that:
\[
    d_{D^2n+D}(\mathcal{Q}, L_2) \leq C_{tr}D^{1/4} d_n(\mathcal{Z}, \mathcal{B}_z)^{1/2}.
\] 
\end{proposition}

\begin{proof}
    Choose $\alpha > 0$ and let $\mathcal{Q}_{\alpha}^{+} = \{ x_i\sigma'(\bm{w}_p\cdot\bm{x}): \norm{\bm{w}_p}_{\infty} \geq \alpha, \ i \in [D], p\in[P] \}$ and $\mathcal{Q}_{\alpha}^{-} = \{ x_i\sigma'(\bm{w}_p\cdot\bm{x}): \norm{\bm{w}_p}_{\infty} <\alpha, \ i \in [D], p\in[P] \}$.

    By \cref{lemma:large}, we have:
    \[
        d_{D^2\cdot n}(\mathcal{Q}^{+}_{\alpha}, L_2) \leq \frac{C_{in}C_w}{\alpha}d_{n}(\mathcal{Z}^{+}_{\alpha}, \mathcal{B}_z) 
        \leq \frac{C_{in}C_w}{\alpha}d_{n}(\mathcal{Z}, \mathcal{B}_z) 
    \]
    and by \cref{lemma:small}, we have:
    \[
        d_{D}(\mathcal{Q}_{\alpha}^{-}, L_2) \leq C_s D^{1/2}\alpha.
    \]
    so using \cref{lemma:combine} to combine the bounds,
    \[
        d_{D^2n+ D}(\mathcal{Q}, L_2) 
        \leq 
        \max{
            \{ C_sD^{1/2}\alpha, C_{in}C_{w}d_{n}(\mathcal{Z}, \mathcal{B}_z)\alpha^{-1}\} 
        }
    \]
    optimising over $\alpha$, gives 
    \[
    \alpha^*=\sqrt{\frac{C_{in}C_wd_{n}(\mathcal{Z}, \mathcal{B}_z)}{C_sD^{1/2}}}
    \]
    and plugging $\alpha^*$ into the combined bound gives the result with $C_{tr}\defeq \sqrt{C_{in}C_wC_s}$.
\end{proof}

\begin{corollary}
    Suppose all the weights are non-negative and the dictionary is bounded in $l_1$ (\cref{ass:boundedl1-weight}). Then there exists dimension-free $C_{tr}^{+} <\infty$ such that,
    \[
        d_{Dn+D}(\mathcal{Q}, L_2) \leq C_{tr}^{+}D^{1/4} d_n(\mathcal{Z}, \mathcal{B}_z)^{1/2},
    \]
    where $C_{tr}^{+} = \sqrt{C_{in}C_1C_s}$.
\end{corollary}

We end the section with a lemma to invert the approximation curves, leaving the ambient spaces implicit for readability.
\begin{lemma} \label{lemma:inv-transfer}
    If $d_{n\cdot a}(X) \leq C d_{n}(Y)^{1/2}$ and $R(X, \epsilon)\geq R^*$ then $R(Y,\left(\frac{\epsilon}{C}\right)^2) \geq \left\lceil \frac{R^*}{a} \right\rceil$
\end{lemma}
\begin{proof}
    First note by definition of $R$ we have: for any $m < R^*$ that $d_m(X) >\epsilon$. Then:
    \[
      R(Y,\left(\frac{\epsilon}{C}\right)^2) = \min \{n: d_n(Y) \leq (\frac{\epsilon}{C})^2 \}
    \]
    Since for any $n < R^*/a$, we have $d_n(Y) \geq (\frac{d_{n \cdot a}(X)}{C})^2 > (\frac{\epsilon}{C})^2$ and $n$ cannot be minimal, we must have $R(Y,\left(\frac{\epsilon}{C}\right)^2) \geq \left\lceil \frac{R^*}{a} \right\rceil$.
\end{proof}

\section{SQ Bounds on Width}
We recall the standard model of a correlational-statistical query (CSQ) oracle, and SQ Dimension based lower bounds.

\begin{definition}[CSQ oracle] For a target $f$ a query $q:\mathcal{X}\rightarrow [-1, 1]$ and a tolerance parameter $\tau$, the CSQ oracle $\operatorname{CSQ}(q,\tau)$ returns a value in the range:
\[
    [\mathbb{E}[q(x)f(x)] - \tau, \mathbb{E}[q(x)f(x)]+\tau].
\]
    
\end{definition}

\begin{definition}[SQ dimension]
The quantity \(\SQdim(\mathcal F,\gamma)\) is the largest integer \(d\)
for which there exist
\(h_1,\ldots,h_d\in\mathcal F\) satisfying
\[
    \|h_i\|_{L_2}=1,
    \qquad
    |\langle h_i,h_j\rangle|
    \leq\gamma
    \quad(i\neq j).
\]
\end{definition}

We recall a standard argument for lower bounding CSQ algorithms, which we will adapt to work with width.

We will consider measuring performance under the total squared loss:
\[
    \mathcal{L}_2(f, f') \defeq \norm{f - f'}^2_{L_2}
\]

\begin{lemma} Suppose a concept class has $\SQdim(\mathcal{F}, \gamma)=d$ then the number of queries, to learn the concept class to squared loss smaller than $\frac{1-\gamma}{2}$ any CSQ algorithm must make:
\[
    M \geq \frac{(d-1)\tau^2}{1 + (d-1)\gamma},
\]
calls to the query oracle.
\end{lemma}

\begin{proof}
    Let $\{h_i\}_{i=1}^d \in \mathcal{F}$ satisfy the claim, and consider,
    \[
        a(q) \defeq (\ip{q}{h_1}, \cdots, \ip{q}{h_d}) \in \mathbb{R}^d.
    \]
    We show $\norm{a}_2^2 \leq 1 + (d-1)\gamma$, and since this mass must be distributed among all the hypothesis, a single query can only eliminate a small number of hypotheses, leading to the desired bound.

    We first write the orthogonal projection of a query, with $\norm{q}_{L_2}\leq 1$ onto the space spanned by the hypotheses \( Pq = \sum_j c_j h_j\) and note:
    \[
        1 \geq \norm{Pq}_2^2 = \sum_{i,j} c_i c_j \ip{h_i}{h_j} = c^t G c,
    \]
    where $G_{ij}\defeq \ip{h_i}{h_j}$. Now by Gershgorin's circle theorem: $\lambda(G)\in [1 - (d-1)\gamma, 1 + (d-1)\gamma]$. Observing $a = Gc$,
    \[
        \frac{\norm{a}^2_2}{\lambda_{\max}(G)}\leq a^tG^{\dagger}a = c^tGG^{\dagger}Gc = c^t G c\leq 1,
    \]
    therefore,
    \[
        \norm{a}^2_2 \leq {1 + (d-1)\gamma}.
    \]
    Now note the number of detectable hypotheses by a single query is $\abs{\{i: \abs{\ip{q}{h_i}}\geq \tau\}} \leq \frac{1 + (d-1)\gamma}{\tau^2}$. Therefore the total number of detectable hypotheses is also bounded:
    \[
        \abs{\{i: \exists m \leq M, \abs{\ip{q_m}{h_i}}\geq \tau\}} \leq M\frac{1+(d-1)\gamma}{\tau^2}.
    \]
    Since we require eliminating $(d-1)$ hypotheses:
    \[
        M \geq \frac{\tau^2(d-1)}{1+(d-1)\gamma}.
    \]
    With fewer queries, a CSQ oracle that answers \(0\) identically leaves at least two hypotheses (w.l.o.g. $h_1, h_2$), indistinguishable so the learner produces the same $f$ under either target. For any $f\in L_2(p_x)$:
    \[
         \norm{h_1 - h_2}_{L_2} \leq \norm{h_1 - f}_{L_2} + \norm{h_2 - f}_{L_2} \leq 2 \max_i \norm{h_i - f}_{L_2}, 
    \]
    but, we also have:
    \[
         \norm{h_1 - h_2}^2_{L_2} \geq 2(1 - \gamma).   
    \]
    Therefore
    \[
               \max_i \mathcal{L}_2(f, h_i) \geq \frac{1-\gamma}{2}.
    \]
    Thus the learner cannot achieve loss lower than $\frac{1- \gamma}{2}$,
    independently of the queries proposed by the learner.
\end{proof}

\subsection{Intrinsic Complexity bounds from CSQ bounds}
The previous proof bounds the number of queries, we would instead like to bound the intrinsic complexity of the queries. This allows us to assess whether an algorithm uses the information it requests efficiently, even if it is inefficient in the way it requests that information perhaps as a result of some un-modelled constraint.

\begin{lemma}[Approximating Transcripts]\label{lemma:transcript-approx}
    Given a concept class $\mathcal{F}$, $d=\SQdim(\mathcal{F}, \gamma)$, and let \[
    \mathcal{T_{A}} = \{Q = (q_1,.., q_M): \text{$Q$ is a query transcript realised by $\mathcal{A}$} \},
    \]
    and suppose
    \[ \sup_{q\in\mathcal Q\subset \mathcal{T_A}} \norm{q}_{L_2(p_x)} \leq B_q. \]
    
    then any CSQ algorithm $\mathcal{A}$ learning to squared loss less than $\frac{1-\gamma}{2}$ must have:
    \[
        \max_{Q \in \mathcal{T_A}}R_{L_2}(Q, \epsilon)  \geq \frac{(d-1) (\tau - \epsilon)^2}{B_q^2[1 + (d-1)\gamma]}
    \]
    for $\epsilon <\tau$.
\end{lemma}

\begin{proof}
Let $\mathcal{H} = \{h_i\}_{i=1}^d \in \mathcal{F}$ satisfy the claim.
First suppose $\{v_1,..., v_{R}\}\subset L_2(p_x)$ are orthonormal, and let $v=\sum_r \alpha_r v_r$ then by Cauchy-Schwartz (in $\mathbb{R}^r$ not $L_2(p_x)$):
\[
    \abs{\ip{v}{h}} \leq \norm{\alpha}_2\sqrt{s_h}, \quad s_h\defeq \sum_{r}\ip{v_r}{h}^2.
\]    
Further,
\[
    \sum_{h\in\mathcal{H}}s_{h} = 
     \sum_r \norm{{a(v_r)}}_2^2, 
     \quad a(v_r)\defeq (\ip{v_r}{h_1}),\cdots \ip{v_r}{h_d}) \in \mathbb{R}^d. 
\]
applying the logic from the previous proof:
\[
    \sum_{h\in\mathcal{H}} s_h \leq R[1+ (d-1)\gamma]
\]
and so the number of hypotheses with large response is bounded:
\[
    \abs{\left\{h: \exists v \in V,\   \abs{\ip{v}{h}}\geq \tau \right\}}
    \leq \abs{\left\{h: s_h \geq \frac{\tau^2}{\norm{\alpha}_2^2}\right\}} \leq \frac{R[1 + (d-1)\gamma] \norm{\alpha}^2_2}{\tau^2}. 
\]
Now suppose $d_{R}(\mathcal{Q}, L_2)\leq\epsilon$ then by assumption every $q\in\mathcal{Q}$ satisfies $\norm{q}_{L_2}^2 \leq B_q^2$ and we have the orthogonal decomposition:
\[
    q = v + e, \quad v \in V, \quad e \in V^{\perp}, \quad \norm{e}_{L_2} \leq \epsilon. 
\]
Then by Pythagoras,
\[
    \norm{\alpha}_2^2 = \norm{v}_{L_2}^2 =  \norm{q}_{L_2}^2 - \norm{e}_{L_2}^2 \leq B_q^2,
\]
and by CS,
\[
    \abs{\ip{q}{h}} \geq \tau \implies \abs{\ip{v}{h}}\geq \tau - \epsilon,
\]
and we have shown the maximum number of hypotheses that can satisfy this is:
\begin{align*}
    \abs{\left\{h: \exists q \in \mathcal{Q} \ \text{s.t.} \ \abs{\ip{q}{h}}     \geq \tau\right\}} 
    \leq &\abs{\left\{h: \exists v \in V,\ \norm{v}_{L_2} \leq B_q, \  \abs{\ip{v}{h}}\geq \tau - \epsilon \right\}} \\ 
    \leq &\frac{RB_q^2[1 + (d-1)\gamma]}{(\tau-\epsilon)^2}. 
\end{align*}
again, if the quantity on the left is less than $d-1$ then a CSQ oracle that answers \(0\) identically leaves at least two hypotheses indistinguishable which by the previous argument result in a loss above the tolerance. Therefore we must have:
\[
    R \geq \frac{(d-1)(\tau - \epsilon)^2}{B_q^2[1+ (d-1)\gamma]},
\]
provided $\epsilon < \tau$.
\end{proof}

\begin{lemma}[Initial-width obstruction for zero-respecting learners]
\label{lemma:initial-width-obstruction}
Let \( d=\SQdim(\mathcal F,\gamma)\)
and let \(\mathcal Z^0\) denote the initial set of neuron activities and suppose
\[
    B_q \geq \sup_{\substack{ \bm w\in\mathcal W\\ i\in[D] }}\norm{x_i\sigma'(\bm{w}\cdot\bm{x})}_{L_2(p_x)}.
\]
Consider a learner whose first-layer weights evolve according to
\[
    W^{t+1}
    =
    W^t+\phi(W^t,\widehat G^t),
\]
where the update rule is zero-respecting:
\[
    \phi(W,0_{P\times D})
    =
    0_{P\times D}
    \qquad
    \text{for every }W.
\]
If the learner achieves squared loss below $\frac{1-\gamma}{2}$ uniformly over \(\mathcal F\), then, for
 \(0\leq\epsilon<\tau\),
\[
    R_{\mathcal B_z}
    \left(
        \mathcal Z_0,
        \frac{\epsilon^2}{C_{tr}^2D^{1/2}}
    \right)
    \geq
    \frac{
        (d-1)(\tau-\epsilon)^2
    }{
        C_qB_q^2D^2[1+(d-1)\gamma]
    }.
\]
\end{lemma}

\begin{proof}
We prove the contrapositive. Suppose that
\[
    R_{\mathcal B_z}
    \left(
        \mathcal Z_0,
        \frac{\epsilon^2}{C_{tr}^2D^{1/2}}
    \right)
    <
    \frac{
        (d-1)(\tau-\epsilon)^2
    }{
        C_qB_q^2D^2[1+(d-1)\gamma]
    }.
\]

By the width-transfer result, the corresponding initial tangent-query
family \(\mathcal Q^0\) has intrinsic dimension strictly below
\[
    R_{L_2(p_x)}
    \left(
        \mathcal{Q}_0,
        \epsilon
    \right)
    < \frac{
        (d-1)(\tau-\epsilon)^2
    }{
        B_q^2[1+(d-1)\gamma]
    }.
\]
The SQ-dimension argument therefore implies that fewer than \(d-1\)
concepts in an SQ-witness family
\[
    \mathcal H=\{h_1,\ldots,h_d\}\subseteq\mathcal F
\]
can produce a query response whose magnitude exceeds \(\tau\).
Consequently, there exist at least two distinct targets
\(h_i,h_j\in\mathcal H\) such that, for every query
\(q\in\mathcal Q_0\),
\[
    \abs{\ip{q}{h_i}}\leq\tau
    \qquad\text{and}\qquad
    \abs{\ip{q}{h_j}}\leq\tau.
\]

For either target, the coordinatewise tolerance-\(\tau\) oracle may
therefore return the all-zero feedback matrix:
\(
    \widehat G^0=0_{P\times D}.
\)
Since the learner is zero-respecting,
\[
    W^1
    =
    W^0+\phi(W^0,0_{P\times D})
    =
    W^0.
\]

The representation and query family are unchanged.
Thus, under either target, the oracle may return zero again at the
next round. Inductively,
\[
    W^t=W^0
    \qquad\text{and}\qquad
    \widehat G^t=0_{P\times D}
    \qquad
    \text{for every }t.
\]

Hence the complete oracle transcript is identical under the two
distinct targets \(h_i\) and \(h_j\). A target-independent learner must
therefore produce the same output under both targets and cannot
identify both correctly. Thus it cannot identify every target in
\(\mathcal F\), proving the contrapositive.
\end{proof}

\begin{corollary} If additionally \cref{ass:boundedl1-weight} holds, and
\[
    \abs{\ip{q_{\bm w, i}}{h}}=\abs{\ip{q_{\abs{\bm w},i}}{h}}, \quad \forall \bm w\in \mathcal W, \ \forall h \in \mathcal{F}
\]
then any zero-respecting learner achieving squared loss below $\frac{1-\gamma}{2}$ uniformly over \(\mathcal F\), has
\[
    R_{\mathcal B_z}
    \left(
        \mathcal Z_0^{+},
        \frac{\epsilon^2}{(C_{tr}^{+})^2D^{1/2}}
    \right)
    \geq
    \frac{
        (d-1)(\tau-\epsilon)^2
    }{
        C_q^{+}B_q^2D[1+(d-1)\gamma]
    }.
\]    
\end{corollary}
\begin{proof}
    By the symmetry assumption, a target $h$ is compatible with the all zero-response for the signed queries $\{q_{\bm w,i}\}$ if and only if it is compatible with queries $\{q_{\abs{\bm{w}},i}\}$. Thus we can complete the analysis with the symmetrised neurons and use the tighter transfer (Lemma~\ref{lemma:unsigned-transfer}).
\end{proof}

\newpage 

\section{Upper Bounds}

\subsection{Forward Approximation}
Here we demonstrate the existence of low width neural networks to approximate sparse parities.

We first recall a simple construction when the dictionary elements are threshold functions. Let:
\[
    \threshold{x} = \begin{cases}
        1, &x\geq 0 \\
        -1,  &x < 0
    \end{cases}
\]
\begin{lemma} [Threshold Construction of k-Parity]
    \label{lemma:k-parity-construction}
    There exists an activity set,
    \[
        \mathcal{Z}=\{\threshold{\bm{w}_p \cdot \bm{x} + b_p}: p \in [k+1]\}
    \]
    and second layer weights $\bm{a} \in \mathbb{R}^{k+1}$ with $\norm{a}_2 = \sqrt{k+1}$
\end{lemma}

\begin{proof}
    First define $t=\sum_{i\in S} x_i$, and note that $\chi_S(x)=(-1)^{(k-t)/2}$, where $t \in \{-k, -k + 2, \dots, k \}$. So let 
    \[
    t_r = -k + 2r,  \quad r \in \{0,\dots,k\},
    \]
    enumerate the grid. 
    
    Now we choose $\bm{w}_p = \bm{1}_S$ to compute the sum, and the biases
    \[
        b_p = k+1-2p, \quad p \in \{1,\dots,k+1\},
    \]
    Therefore the interpolating matrix is $B_{rp} \defeq \threshold{t_r + b_p}$. Which is given by,
    \[ 
    B_{rp} = \begin{cases} 1, & p\leq r,\\ -1, & p>r. \end{cases}
    \]
    Suppose $y=Ba$ then we have:
    \[
        a_r = \frac{y_{r} - y_{r-1}}{2} \ \text{for} \ r\in \{1,...,k\}, \qquad a_{k+1} = -\frac{y_0 + y_k}{2}.
    \]
    For parity $y_{r}=(-1)^{k-r}$ so $\abs{a_r}\leq 1$.
    We therefore have $\norm{a}_2 \leq \sqrt{k+1}$.    
\end{proof}

We now use a perturbation technique to show the existence of a well conditioned basis of finite weight $\text{tanh}$ functions.

\begin{lemma}
    There exists an activity set,
    \[
        \mathcal{Z}=\{\tanh(\bm{w}_p \cdot \bm{x} + b_p): p \in [k+1]\}
    \]
    and second layer weights $\bm{a} \in \mathbb{R}^{k+1}$ with $\norm{a}_2 = O(\sqrt{k+1})$
\end{lemma}
\begin{proof}
    First note for $u \geq 0$,
    \[
        0 \leq 1 - \tanh(\lambda u ) = \frac{2}{e^{2\lambda u} + 1} \leq 2e^{-2\lambda u} 
    \]

    Consider $A_{rp}(\lambda) = \tanh(\lambda(t_r + b_p))$ and consider for $r > p$:
    \[
        E_{rp} = B_{rp} - A_{rp} = 1 - \tanh(\lambda(t_r + b_p)) \leq 2e^{-2 \lambda [1 + 2(r-p)]}
    \]
    and by symmetry: $\abs{E_{rp}} \leq 2e^{-2\lambda \abs{1 + 2(r-p)}}$. Therefore every column has,
    \[
        \sum_{p} \abs{E_{rp}} \leq 2\sum_{m \in \mathbb{Z}} e^{-2\lambda\abs{1+2m}} = \frac{4e^{-2\lambda}}{1 - e^{-4\lambda}}\defeq \delta_{\gamma}.
    \]
    So $\norm{E}_1 \leq \delta_{\gamma}$, applying Weyl's inequality:
    \[
        \sigma_{\min}(A) \geq \sigma_{\min}(B) - \norm{E}_2 \geq 1- \delta_{\gamma}.
    \]
    So for sparse parity we have $\norm{a}_2 \leq \frac{1}{1 - \delta_{\gamma}}\sqrt{k+1}$. 
\end{proof}

\begin{corollary}
    There exists a backwards dictionary,
    \[
        \mathcal{T}=\{\sech^2(\bm{w}_p \cdot \bm{x} + b_p): p \in [k+1]\}
    \]
    and second layer weights $\bm{a} \in \mathbb{R}^{k+1}$ with $\norm{a}_2 = O(\sqrt{k+1})$
\end{corollary}

The same finite-grid interpolation argument applies to derivative features, except with convergence towards Kronecker deltas. 
Let \(t_r=-k+2r\), \(r=0,\ldots,k\), and define 
    \[ A_{rp}(\lambda) \defeq \sech^2 \left( \lambda(t_r-t_p) \right). 
\] 
Since \(A_{rr}(\lambda)=1\) and 
\[ 
    \sech^2(u) \leq 4e^{-2\abs{u}}, 
\] 
the off-diagonal entries satisfy 
\[ 
    \abs{A_{rp}(\lambda)} \leq 4e^{-4\lambda\abs{r-p}}, \qquad r\neq p. 
\] 
Consequently, 
\[ 
    \norm{ A(\lambda)-I }_2 \leq \norm{ A(\lambda)-I }_\infty \leq \frac{ 8e^{-4\lambda} }{ 1-e^{-4\lambda} } \defeq \delta_\lambda. 
\] 
Choosing \(\lambda\) so that \(\delta_\lambda<1\) gives \[ \sigma_{\min} \left( A(\lambda) \right) \geq 1-\delta_\lambda. \] Hence every function on the \(k+1\) possible values of \(\sum_{i\in S}x_i\), and in particular \(k\)-parity, is represented by \(k+1\) such derivative features with coefficient norm \[
    \norm{\bm a}_2 \leq \frac{\sqrt{k+1}}{1-\delta_\lambda} = O(\sqrt{k+1}).
\]

\begin{lemma} [Walsh coefficients of Sum Neuron]
    Let $m \geq 1$ be fixed, and $U\subset[D]$ with $\abs{U}=s$. Define,
    \[
        h_U(\bm{x}) \defeq \sech^2(\frac{1}{\sqrt{s}}\sum_{i \in U} x_i - b_m)
    \]
    For any $T\subset U$ with $\abs{T}=m$, and define
    \(
        \lambda_{m,s} \defeq \widehat{h_U}(T)=\ip{h_U}{\chi_{T}}.
    \)
    There exists $b_m$ and a constant $c_m \neq 0$ such that,
    \[
        \lambda_{m,s} = c_ms^{-m/2} + o(s^{-m/2}).
    \]
    In particular, for any $S\subset [D]$ with $\abs{S}=k$,
    \[
        \ip{x_d h_U}{\chi_S} = \begin{cases}
            c_{k-1}s^{-(k-1)/2} + o(s^{-(k-1)/2}), &d \in S, S \setminus \{d\} \subseteq U \\
            O(s^{-(k+1)/2}), &d \notin S, S \cup \{d\} \subseteq U \\
            0 & \text{otherwise}.
        \end{cases}
    \]
    
\end{lemma}
\begin{proof}
    Let $R \defeq \sum_{i \in U \setminus T} x_i$ and let $f(u)=\sech^2(u -b)$. Conditioning on coordinates outside of $T$:
    \[
        \lambda_{m,s} = \mathbb{E}_R \mathbb{E}_{\bm{x}_T}\left[ \chi_T(\bm{x}) f\left(\frac{R + \sum_{i\in T}x_i}{\sqrt{s}}\right) \right]
    \]
    Let the central difference operator be:
    \[
        D_h f(u) \defeq \frac{f(u + h) - f(u-h)}{2}
    \]
    which can be expressed using a single Rademacher variable $\epsilon$, $D_h f(u) = \mathbb{E}_{\epsilon}[\epsilon f(u + \epsilon h)]$, applying this identity successively with $h=s^{-1/2}$ yields:
    \[
        \lambda_{m,s} = \mathbb{E}_R\left[D^m_{s^{-1/2}}f \left(\frac{R}{\sqrt{s}} \right)\right].
    \]
    Since $f$ has bounded derivatives at all orders,
    \[
        D^m_h f(u) = h^m f^{(m)}(u) + O(h^{m+2}),
    \]
    uniformly in $u$. Hence,
    \[
        \lambda_{m,s} = s^{-m/2}\mathbb{E}_R\left[f^{(m)}\left(\frac{R}{\sqrt{s}}\right)\right] + O(s^{-(m+2)/2}).
    \]
    Now $R$ is the sum of $s-m$ independent Rademacher variables, so
    \[
        \frac{R}{\sqrt{s}} \rightarrow^{d} Z \sim N(0, 1)
    \]
    because $f^{(m)}$ is bounded and continuous,
    \[
        \mathbb{E}_R\left[f^{(m)}\left(\frac{R}{\sqrt{s}}\right)\right] \rightarrow \mathbb{E}_Z\left[f^{(m)}(Z)\right] = \mathbb{E}[H_m(Z)\sech^2(Z - b)] \defeq c_{m}(b)
    \]
    where $H_m$ is the probabilists' Hermite polynomial.
    
    It remains to choose \(b\) such that \(c_m(b)\neq 0\). Define 
    \[ F(b) \defeq \mathbb{E}\left[\sech^2(Z-b)\right]. \]
    Differentiation under the expectation gives 
    \[ F^{(m)}(b) = (-1)^m c_m(b). \]
    The function \(F\) is positive and tends to zero as \(\abs{b}\rightarrow\infty\). In particular, \(F\) is not a polynomial of degree at most \(m-1\), and hence \(F^{(m)}\) cannot vanish identically. It follows that there exists \(b_m\in\mathbb{R}\) such that \[ c_m(b_m)\neq 0. \]
\end{proof}

\subsubsection{Covering}
The required pool family is a classical \((D,s,k-1)\)-covering design \citep{gordonCoverings2006}. Denoting its minimum size by the covering number \(C(D,s,k-1)\), we have \[ C(D,s,k-1) \geq \frac{\binom D{k-1}}{\binom s{k-1}}, \] while standard covering-design bounds give  
\[ 
    C(D,s,k-1) \leq \frac{\binom D{k-1}}{\binom s{k-1}} \left( 1+\log\binom s{k-1} \right).
\]
In our regime, \(k\) is fixed and \(s=\Theta(D^\alpha)\), so
\(
    \log\binom s{k-1}
    =
    \Theta_k(\log D).
\)
The sharper covering-design bound therefore does not remove the
logarithmic factor from our asymptotic width estimate. We retain the
elementary probabilistic proof below for completeness.

\begin{lemma}[Covering family of sparse pools] 

\label{lemma:sparse-pool-cover} Let \(k\leq s<D\). There exists a family 
\[ 
    \mathcal U \subseteq \left\{ U \subset [D]: |U|=s \right\} 
\] 
of cardinality \[ |\mathcal U| = O_k \left( \left( \frac Ds \right)^{k-1} \log D \right) \] such that every \((k-1)\)-subset \(A\subseteq[D]\) is contained in at least one \(U\in\mathcal U\). 
\end{lemma} 

\begin{proof} Draw \(P\) subsets \(U_1,\ldots,U_P\) independently and uniformly from the size-\(s\) subsets of \([D]\). For a fixed \(A\subseteq[D]\), with \(|A|=k-1\), \[ q_{D,s,k} \defeq \Pr(A\subseteq U_p) = \frac{ \binom{D-k+1}{s-k+1} }{ \binom Ds } = \frac{ \binom{s}{k-1} }{ \binom{D}{k-1} }. \] Hence \[ \Pr \left( A \not\subseteq U_p \text{ for every }p\in[P] \right) = (1-q_{D,s,k})^P \leq \exp(-Pq_{D,s,k}). \] Taking a union bound over the \(\binom{D}{k-1}\) possible sets \(A\) gives \[ \Pr \left( \exists A,\ |A|=k-1: A\not\subseteq U_p \text{ for every }p \right) \leq \binom{D}{k-1} \exp(-Pq_{D,s,k}). \] This probability is strictly smaller than one whenever \[ P > \frac{ \log\binom{D}{k-1} }{ q_{D,s,k} }. \] Therefore, a covering family exists with \[ \begin{aligned} P &= O_k \left( \frac{ \binom{D}{k-1} }{ \binom{s}{k-1} } \log D \right) \\ &= O_k \left( \left( \frac Ds \right)^{k-1} \log D \right). \end{aligned} \]

\end{proof}

\begin{theorem}[One-round sparse-parity upper bound]
\label{thm:one-round-parity-upper-full}
Let \(k\) be fixed, and suppose \(k\leq s<D\) satisfies
\[
    2\tau
    <
    \min
    \left\{
        \abs{\lambda_{k-1,s}},
        \abs{
            \lambda_{k-1,s}
            -
            \lambda_{k+1,s}
        }
    \right\}.
\]
Then there exists a network of width
\[
    P
    =
    O_k
    \left(
        \left(
            \frac Ds
        \right)^{k-1}
        \log D
    \right)
\]
whose coordinatewise gradient transcript identifies every
\(k\)-sparse parity in one non-adaptive oracle round.
\end{theorem}

\begin{proof}
Let
\[
    \mathcal U
    =
    \{U_1,\ldots,U_P\}
\]
be the covering family supplied by
\cref{lemma:sparse-pool-cover}. For each pool \(U_p\), initialise one
pooled \(\tanh\) neuron whose derivative feature is
\[
    h_{U_p}(\bm x)
    \defeq
    \sech^2
    \left(
        \frac{1}{\sqrt s}
        \sum_{j\in U_p}x_j
        -
        b_{k-1}
    \right).
\]

Let the unknown target be \(\chi_S\), where \(|S|=k\). The population
response of pool \(U_p\) in coordinate \(d\) is
\[
\begin{aligned}
    g_{p,d}(S)
    &\defeq
    \ip{
        x_dh_{U_p}
    }{
        \chi_S
    }
    \\
    &=
    \widehat h_{U_p}
    \left(
        S\mathbin{\triangle}\{d\}
    \right).
\end{aligned}
\]
By permutation symmetry of \(h_{U_p}\), this response satisfies
\[
    g_{p,d}(S)
    =
    \begin{cases}
        \lambda_{k-1,s},
        &
        d\in S
        \text{ and }
        S\setminus\{d\}\subseteq U_p,
        \\[1mm]
        \lambda_{k+1,s},
        &
        d\notin S
        \text{ and }
        S\cup\{d\}\subseteq U_p,
        \\[1mm]
        0,
        &
        \text{otherwise}.
    \end{cases}
\]
For every target coordinate \(d\in S\), the set \(S\setminus\{d\}\)
has cardinality \(k-1\). By the covering property, there exists some
pool \(U_p\) containing this set. Consequently,
\[
    g_{p,d}(S)
    =
    \lambda_{k-1,s}.
\]
Whereas if \(d\notin S\), every response \(g_{p,d}(S)\) belongs to \(
    \left\{
        0,
        \lambda_{k+1,s}
    \right\}
\).
The oracle returns
\[
    \widehat g_{p,d}(S)
    =
    g_{p,d}(S)+\xi_{p,d},
    \qquad
    \abs{\xi_{p,d}}
    \leq
    \tau.
\]
The condition
\[
    2\tau
    <
    \min
    \left\{
        \abs{\lambda_{k-1,s}},
        \abs{
            \lambda_{k-1,s}
            -
            \lambda_{k+1,s}
        }
    \right\}
\]
ensures that the tolerance interval centred at
\(\lambda_{k-1,s}\) is disjoint from those centred at \(0\) and
\(\lambda_{k+1,s}\).

The support can therefore be recovered by declaring
\(
    d\in\widehat S
\)
if and only if there exists \(p\in[P]\) such that
\(
    \abs{
        \widehat g_{p,d}(S)
        -
        \lambda_{k-1,s}
    }
    \leq
    \tau
\)
Every \(d\in S\) is accepted by its witness pool, while no
\(d\notin S\) can be accepted. Hence \(\widehat S=S.\)
All pool-coordinate responses are requested simultaneously, so the
procedure is non-adaptive and uses one oracle round.
\end{proof}

\subsection{Explicit Endpoints}\label{sec:endpoints}

\begin{lemma}
    If $\tau, k$ are fixed and $\tau < \frac{1}{2Ck}$ where $C>0$ only depends on $C_w$, then there is a network with,
    \[
        P = O_k(D^{k-1})
    \]
    that identifies every \(k\)-sparse parity with one non-adaptive round (a single call to the oracle).
\end{lemma}
\begin{proof}
    By our previous construction there are neural networks of $\sech^2$ units with width $k$ and $\norm{\bm{a}_{T}}_2\leq C\sqrt{k}$ to represent each parity $\chi_{T}$ on $\abs{T}=k-1$ coordinates. Therefore using the hidden units in parallel, gives a network of width $k {D\choose k-1}$. To decode we combine responses of each subnetwork using the identified second layer weights $\bm{a}_T$, allowing us to decode:
    \[
        \ip{\chi_{T \Delta \{d\}}}{\chi_S} + \xi, \quad \abs{\xi}\leq C k \tau
    \]
    Since $\ip{\chi_{T\Delta\{d\}}}{\chi_S} = 1[S = T\Delta \{d\}]$, $\tau < \frac{1}{2Ck}$ suffices to distinguish a $0$ response from a hit, and thus we can identify the target parity.
\end{proof}

\begin{lemma}[Single-neuron identification] \label{lemma:single-neuron-identification} Let \(k\) be fixed, and consider the full-support derivative feature
\[ h_D(\bm{x}) = \sech^2\left( \frac{1}{\sqrt D}\sum_{i=1}^D x_i-b_{k-1} \right). \] 
Define \[ \Delta_{D,k} \defeq \lambda_{k-1,D}-\lambda_{k+1,D}. \] 

If \( 2\tau<\abs{\Delta_{D,k}}, \) then this single neuron identifies every \(k\)-sparse parity in one non-adaptive round. 
In particular, for fixed \(k\), there exists \(c_k>0\) such that a single neuron suffices whenever \[ \tau<c_kD^{-(k-1)/2} \] for all sufficiently large \(D\).
\end{lemma}

\begin{proof}
Let the unknown target be \(\chi_S\), where \(\abs{S}=k\). For each coordinate \(d\in[D]\), the population-gradient response is
\[ 
    g_d(S) \defeq \ip{x_dh_D}{\chi_S} = \ip{h_D}{\chi_{S\triangle\{d\}}}. 
\]
Since \(h_D\) is invariant under coordinate permutations, its Walsh coefficient depends only on the order of the corresponding Walsh function. Therefore, 
\[ 
    g_d(S) = \begin{cases} \lambda_{k-1,D}, & d\in S,\\[1mm] \lambda_{k+1,D}, & d\notin S. \end{cases} 
\] 
The coordinate-wise gradient oracle returns 
\[ 
    \widehat g_d(S) = g_d(S)+\xi_d, \qquad \abs{\xi_d}\leq\tau. 
\] 
The assumption \[ 2\tau < \abs{ \lambda_{k-1,D}-\lambda_{k+1,D} } \] ensures that these intervals are disjoint. The learner can therefore recover \(S\) by thresholding the coordinates at 
\[ 
    \theta_{D,k} \defeq \frac{ \lambda_{k-1,D}+\lambda_{k+1,D} }{2}, 
\] 
with the direction of the threshold determined by the sign of \(\Delta_{D,k}\). Finally, the Walsh-coefficient estimate gives 
\[ 
    \lambda_{k-1,D} = c_{k-1}D^{-(k-1)/2} + o\left(D^{-(k-1)/2}\right), \quad \lambda_{k+1,D} = O(D^{-(k+1)/2})
\] 
where \(c_{k-1}\neq0\),  thus, for sufficiently large \(D\), 
\[ 
\abs{\Delta_{D,k}} \geq \frac{\abs{c_{k-1}}}{2} D^{-(k-1)/2}. 
\] 
For example, the sufficient condition \[ \tau < \frac{\abs{c_{k-1}}}{4} D^{-(k-1)/2}
\] 
guarantees exact recovery.
\end{proof}

\newpage

\section{Generalisation and Interpolation}

\subsection{Variation Balls}

\begin{proposition}[Rademacher complexity of a variation-norm ball]
\label{prop:variation-ball-rademacher}
Let
\[
    \mathcal B_\phi(A)
    \defeq
    \left\{
        f\in\mathcal B_\phi:
        \norm{f}_{\mathcal B_\phi}
        \leq
        A
    \right\}.
\]
be the variation ball of radius A.

For any fixed sample
\(
    \bm X
    =
    (X_1,\ldots,X_N),
\)
the empirical Rademacher complexity satisfies
\[
\begin{aligned}
    \widehat{\mathfrak R}_{\bm X}
    \left(
        \mathcal B_\phi(A)
    \right)
    &\leq
    \frac{A}{N}
    \mathbb E_{\bm\varepsilon}
    \left[
        \sup_{\bm w\in\mathcal W}
        \left|
            \sum_{i=1}^N
            \varepsilon_i
            \phi_{\bm w}(X_i)
        \right|
    \right].
\end{aligned}
\]
Equivalently,
\[
    \widehat{\mathfrak R}_{\bm X}
    \left(
        \mathcal B_\phi(A)
    \right)
    \leq
    A\,
    \widehat{\mathfrak R}_{\bm X}
    \left(
        \Phi_\pm
    \right),
\]
where
\[
    \Phi_\pm
    \defeq
    \Phi\cup(-\Phi).
\]
\end{proposition}

\begin{proof} 
Fix \(f\in\mathcal B_\phi(A)\) and \(\epsilon>0\). By the definition of the variation norm, \(f\) admits a representation 
\[ 
    f = \int_{\mathcal W} \phi(\bm{w}, \cdot) \,d\mu 
\] 
with \(
\abs{\mu}_{\mathrm{TV}} \leq A+\epsilon. 
\) 
Therefore, 
\[ 
\begin{aligned} 
    \sum_{i=1}^N \varepsilon_i f(X_i) 
    &= \int_{\mathcal W} \left( \sum_{i=1}^N \varepsilon_i \phi(\bm w, X_i) \right) d\mu \\ 
    &\leq \abs{\mu}_{\mathrm{TV}} \sup_{\bm w\in\mathcal W} \left| \sum_{i=1}^N \varepsilon_i \phi_{\bm w}(X_i) \right|. \end{aligned} 
\] 
Taking the supremum over \(f\in\mathcal B_\phi(A)\), averaging over the Rademacher variables, dividing by \(N\), and sending \(\epsilon\downarrow0\) proves the claim.
\end{proof}

Let  $\Phi_p \defeq \{\tanh(\bm{w}\cdot \bm{x}): w \in  B_{p}(0, C_p)\}$ 
\begin{lemma}
    For bounded input $\abs{X_{i,j}} \leq 1$ we have:
    \[
        \widehat{\mathfrak{R}}(\Phi_{1,\pm}) \leq C_1\sqrt{\frac{\log(2D)}{N}}
    \]
    and
    \[
        \widehat{\mathfrak{R}}(\Phi_{\infty,\pm}) \leq C_\infty{\frac{D}{\sqrt N}}
    \]
\end{lemma}

\begin{proof}
Since $\tanh$ is $1$-Lipschitz,
\[
        \frac{1}{N}\mathbb{E}_{\epsilon}\left[\sup_{\bm w \in \mathcal{W}}\abs{\sum_{i=1}^N \epsilon_i \tanh(\bm{w}\cdot X_i)} \right] 
        \leq \frac{1}{N}\mathbb{E}_{\epsilon}\left[\sup_{\bm w \in \mathcal{W}}\abs{\sum_{i=1}^N \epsilon_i \bm{w}\cdot X_i} \right] 
\]
By norm duality
\[
    \begin{aligned}
        \frac{1}{N}\mathbb{E}_{\epsilon}\left[\sup_{\bm w \in \mathcal{W}}\abs{\sum_{i=1}^N \epsilon_i \bm{w}\cdot X_i} \right] 
        &= \frac{1}{N}\mathbb{E}_{\epsilon}\left[\sup_{\bm w \in \mathcal{W}}\abs{\bm{w} \cdot \sum_{i=1}^N \epsilon_i  X_i} \right] \\
        &= \frac{C_1}{N}\mathbb{E}_{\epsilon}\left[\norm{\sum_{i=1}^N \epsilon_i  X_i}_{\infty} \right] \\
        &\leq C_1\sqrt{\frac{2\log D}{N}}
    \end{aligned}
\]
where the last step is via a sub-gaussian maximal inequality.

For the second claim
\[
    \begin{aligned}
        \frac{1}{N}\mathbb{E}_{\epsilon}\left[\sup_{\bm w \in \mathcal{W}}\abs{\sum_{i=1}^N \epsilon_i \bm{w}\cdot X_i} \right] 
        &= \frac{1}{N}\mathbb{E}_{\epsilon}\left[\sup_{\bm w \in \mathcal{W}}\abs{\bm{w} \cdot \sum_{i=1}^N \epsilon_i  X_i} \right] \\
        &= \frac{C_\infty}{N}\mathbb{E}_{\epsilon}\left[\norm{\sum_{i=1}^N \epsilon_i  X_i}_{1} \right] \\
        &\leq C_{\infty}\frac{D}{\sqrt{N}}
    \end{aligned}
\]
where the final step is by Jensen's inequality.
\end{proof}

\begin{lemma}[Rademacher complexity of a finite-dimensional class]
\label{lemma:finite-dimensional-rademacher}
Let
\(V\subseteq L_2(p_x)\)
be a linear subspace with
\(\dim V\leq R\), and define
\[
    \mathcal F_{V,B}
    \defeq
    \left\{
        f\in V:
        \norm{f}_{L_2(p_x)}
        \leq
        B
    \right\}.
\]
Then
\[
    \mathfrak R_N
    \left(
        \mathcal F_{V,B}
    \right)
    \leq
    B
    \sqrt{
        \frac RN
    }.
\]
\end{lemma}
\begin{proof}
Let
\(u_1,\ldots,u_r\)
be an orthonormal basis of \(V\) in \(L_2(p_x)\). Every
\(f\in\mathcal F_{V,B}\) can be written as
\[
    f
    =
    \sum_{j=1}^r b_ju_j
\]
for some \(\bm b\in\mathbb R^r\) satisfying
\[
    \norm{\bm b}_2
    =
    \norm{f}_{L_2(p_x)}
    \leq
    B.
\]
Therefore,
\[
\begin{aligned}
    \mathfrak R_N
    \left(
        \mathcal F_{V,B}
    \right)
    &=
    \frac1N
    \mathbb E_{\bm X,\bm\epsilon}
    \left[
        \sup_{\norm{\bm b}_2\leq B}
        \sum_{i=1}^N
        \epsilon_i
        \sum_{j=1}^r
        b_ju_j(X_i)
    \right]
    \\
    &=
    \frac{B}{N}
    \mathbb E_{\bm X,\bm\epsilon}
    \left\|
        \sum_{i=1}^N
        \epsilon_i
        \bm u(X_i)
    \right\|_2,
\end{aligned}
\]
By Jensen's inequality,
\[
\begin{aligned}
    \mathfrak R_N
    \left(
        \mathcal F_{V,B}
    \right)
    &\leq
    \frac{B}{N}
    \left(
        \mathbb E_{\bm X,\bm\epsilon}
        \left\|
            \sum_{i=1}^N
            \epsilon_i\bm u(X_i)
        \right\|_2^2
    \right)^{1/2}.
\end{aligned}
\]
Cross terms vanish, giving
\[
    \mathbb E_{\bm X,\bm\epsilon}
    \left\|
        \sum_{i=1}^N
        \epsilon_i\bm u(X_i)
    \right\|_2^2
    =
    N
    \sum_{j=1}^r
    \mathbb E
    \left[
        u_j(X)^2
    \right]
    \leq
    NR.
\]
Consequently,
\[
    \mathfrak R_N
    \left(
        \mathcal F_{V,B}
    \right)
    \leq
    B
    \sqrt{
        \frac RN
    }.
\]
\end{proof}

\subsection{Generalisation Bound}

\begin{proposition}[Generalisation from effective-width approximation] \label{prop:width-generalisation} 
Assume that
\(
    M_\sigma
    \defeq
    \sup_{\bm w\in\mathcal W,\bm x\in\mathcal X}
    \abs{\sigma(\bm w\cdot\bm x)}
    <\infty
\).
Let 
\(
    \mathcal Z = \{z_1,\ldots,z_P\} \subseteq \mathcal B_z 
\) 
be fixed independently of an i.i.d.\ sample \( \mathcal S = \{(X_i,Y_i)\}_{i=1}^N. \)
Suppose there exist a subspace \(V\subseteq\mathcal B_z\), with \(\dim V\leq R\), and approximants \(v_p\in V\) satisfying 
\[ 
    \norm{z_p-v_p}_{\mathcal B_z} \leq \delta \qquad \text{for every }p\in[P]. 
\] 
For \(A>0\), define the projected readout class
\[ 
    \mathcal F_{V,A} \defeq \left\{ \sum_{p=1}^P a_pv_p: \norm{\bm a}_1\leq A \right\}. 
\] 
Let \(\ell:\mathcal Y\times\mathbb R\to[0,M]\) be \(L\)-Lipschitz in its second argument. Then, with probability at least \(1-\eta\), simultaneously for every \(\bm a\) satisfying \(\norm{\bm a}_1\leq A\),
\[ 
    \begin{aligned} \abs{ \mathcal L(f_{\bm a}) - \widehat{\mathcal L}_{\mathcal S}(f_{\bm a}) } \leq\;& C L \mathfrak R_N \left( \mathcal F_{V,A} \right) + 2LM_\sigma A\delta \\ &+ CM \sqrt{ \frac{\log(2/\eta)}N }.
    \end{aligned} 
\]
where $\mathfrak R_N$ denotes the expected Rademacher complexity. 
\end{proposition}

\begin{proof}
For each $p$ define the residual feature as an element achieving the approximation guarantee,
\(
    e_p \defeq z_p - v_p,
\)
and define the projected predictor,
\[
    f_{\bm a, V} = \sum_{i=1}^P a_i v_p.
\]
We first show,
\[
\norm{f_{\bm a}-f_{\bm a,V}}_{L_\infty}
    \leq
    M_\sigma A\delta,
\]
and the statement will follow by the Rademacher generalisation in the reduced class and triangle inequality.
First note, 
\[ 
    \norm{e_p}_{\infty} \leq M_{\sigma}\norm{e_p}_{\mathcal B_z} \leq M_{\sigma}\delta
\]
Therefore,
\[
\begin{aligned}
\norm{f_{\bm a}-f_{\bm a,V}}_{L_\infty}
    &= \norm{\sum_{i=1}^Pa_i e_i}_{L_\infty} \\
    &\leq \sum_{i=1}^P \abs{a_i}\norm{e_i}_{L_\infty} \\
    &\leq AM_{\sigma}\delta,
\end{aligned}
\]
demonstrating the claim. As a consequence the following holds pointwise,
\[
    \abs{\ell(\bm{y}, f_{\bm a}(x)) - \ell(y, f_{\bm a, V}(x))} \leq LM_{\sigma}A\delta
\]
because the loss is $L$-Lipschitz in its second argument, consequently the inequality also descends to this form on population and empirical risks.

By the same property, we also have via Rademacher contraction,
\[
    \mathfrak{R}_N(\ell \circ \mathcal{F}_{V,A}) \leq L \mathfrak{R}_N(\mathcal{F}_{V,A}) 
\]
and because the loss is bounded in \([0,M]\), the standard Rademacher
generalisation inequality implies that, with probability at least
\(1-\eta\), simultaneously for every \(f\in\mathcal F_{V,A}\),
\[
    \abs{
        \mathcal L(f)
        -
        \widehat{\mathcal L}_{S}(f)
    }
    \leq
    2L
    \mathfrak R_N(\mathcal F_{V, A})
    +
    3M
    \sqrt{
        \frac{\log(2/\eta)}{2N}
    }.
\]
Thus combining the bounds for any $f_{\bm a}$ with $\norm{\bm a}_1 \leq A$ we have,
\[
    \abs{
        \mathcal L(f_{\bm a})
        -
        \widehat{\mathcal L}_{S}(f_{\bm a})
    }
    \leq
    2L
    \mathfrak R_N(\mathcal F_{V, A})
    + 2LM_{\sigma}A\delta
    +
    3M
    \sqrt{
        \frac{\log(2/\eta)}{2N}
    }.
\]

\end{proof}

\begin{corollary}[Combined effective-width and variation-norm bound] \label{cor:combined-width-generalisation} 
Under the conditions of \cref{prop:width-generalisation}, suppose additionally that 
    \[ \abs{X_j} \leq 1 \qquad \text{a.s. for every }j\in[D]. 
\] Then, with probability at least \(1-\eta\), simultaneously for every readout satisfying \(\norm{\bm a}_1\leq A\), 
\[ 
    \begin{aligned} 
    \abs{ \mathcal L(f_{\bm a}) - \widehat{\mathcal L}_{\mathcal S}(f_{\bm a}) } \lesssim\;
    & \frac{ LA(1+\delta) }{ \sqrt N } \min \left\{ M_\sigma\sqrt R,\, C_1\sqrt{\log(2D)},\, C_wD \right\} \\ 
    &+ LM_\sigma A\delta + M \sqrt{ \frac{\log(2/\eta)}N }. 
    \end{aligned} 
\] Where, \(C_1\) is available when \(\sup_{\bm w\in\mathcal W} \norm{\bm w}_1 \leq C_1, \) and the term involving \(C_w\) is available when \( \sup_{\bm w\in\mathcal W} \norm{\bm w}_\infty \leq C_w. \) The implicit constant is universal.
\end{corollary}

\begin{proof}
Each activity atom satisfies
\(
    \norm{z_p}_{\mathcal B_z}
    \leq
    1.
\)
Hence
\[
    \norm{v_p}_{\mathcal B_z}
    \leq
    \norm{z_p}_{\mathcal B_z}
    +
    \norm{z_p-v_p}_{\mathcal B_z}
    \leq
    1+\delta.
\]
It follows that every
\(f_{\bm a,V}\in\mathcal F_{V,A}\)
satisfies
\[
    \norm{f_{\bm a,V}}_{\mathcal B_z}
    \leq
    A(1+\delta).
\]
Since
\(
    \norm{f}_{L_2(p_x)}
    \leq
    M_\sigma
    \norm{f}_{\mathcal B_z},
\)
and \(\dim V\leq R\), the standard Rademacher estimate for an
\(R\)-dimensional \(L_2(p_x)\)-ball gives
\[
    \mathfrak R_N(\mathcal F_{V,A})
    \leq
    M_\sigma A(1+\delta)
    \sqrt{\frac RN}.
\]
Alternatively, the variation-ball estimate and the Rademacher bounds
for the neural atom class give
\[
    \mathfrak R_N(\mathcal F_{V,A})
    \lesssim
    A(1+\delta)
    C_1
    \sqrt{\frac{\log(2D)}N}
\]
under the uniform \(\ell_1\)-bound, and
\[
    \mathfrak R_N(\mathcal F_{V,A})
    \lesssim
    A(1+\delta)
    C_w
    \frac{D}{\sqrt N}
\]
under the uniform \(\ell_\infty\)-bound.

Taking the smallest available estimate yields
\[
    \mathfrak R_N(\mathcal F_{V,A})
    \lesssim
    \frac{
        A(1+\delta)
    }{
        \sqrt N
    }
    \min
    \left\{
        M_\sigma\sqrt R,\,
        C_1\sqrt{\log(2D)},\,
        C_wD
    \right\}.
\]
Substitution into
\cref{prop:width-generalisation}
proves the result.
\end{proof}



\newpage

\newpage
\subsection{Norm-control}

\begin{lemma}[Residual cost of empirical interpolation]
\label{lemma:residual-interpolation}
Let
\[
    \mathcal Z
    =
    \{z_1,\ldots,z_P\}
    \subseteq
    \mathcal B_z,
\]
and suppose there exists a subspace
\(V\subseteq\mathcal B_z\), with
\(\dim V\leq R\), and
\(v_p\in V\) satisfying
\[
    \norm{z_p-v_p}_{\mathcal B_z}
    \leq
    \delta
    \qquad
    \text{for every }p\in[P].
\]
Let
\(
    \bm{X}
    =
    (X_1,\ldots,X_N)
\)
be any collection of sample inputs, and define the sample-evaluation
space
\[
    V_{\bm{X}}
    \defeq
    \left\{
        \left(
            v(X_1),\ldots,v(X_N)
        \right):
        v\in V
    \right\}
    \subseteq
    \mathbb R^N.
\]
Let
\(\Pi_{V, \bm X}^{\perp}\)
denote the Euclidean orthogonal projection onto
\(V_{\bm X}^{\perp}\) under the empirical inner product.

Suppose the readout
\[
    f_{\bm a}
    =
    \sum_{p=1}^P a_pz_p
\]
interpolates a label vector
\(\bm y\in\mathbb R^N\), so that
\(
    f_{\bm a}(X_i)
    =
    y_i
\) for every $i\in[N]$.
Then
\[
    \norm{
        \Pi_{V, \bm X}^{\perp}\bm y
    }_2
    \leq
    N^{1/2}M_\sigma\delta
    \norm{\bm a}_1,
\]
Consequently, if \(\delta>0\), every interpolating readout satisfies
\[
    \boxed{
        \norm{\bm a}_1
        \geq
        \frac{
            \norm{
                \Pi_{V, \bm X}^{\perp}\bm y
            }_2
        }{
            \sqrt{N}M_\sigma\delta
        }.
    }
\]
\end{lemma}

\begin{proof}
For each \(p\in[P]\), define the residual feature
\(
    e_p
    \defeq
    z_p-v_p
\).
By assumption,
\(
    \norm{e_p}_{\mathcal B_z}
    \leq
    \delta
\).
Since the variation norm controls the uniform norm,
\[
    \norm{e_p}_{L_\infty}
    \leq
    M_\sigma
    \norm{e_p}_{\mathcal B_z}
    \leq
    M_\sigma\delta.
\]
Decompose the predictor as
\[
    f_{\bm a}
    =
    f_{\bm a,V}
    +
    f_{\bm a,E},
\]
where
\[
    f_{\bm a,V}
    \defeq
    \sum_{p=1}^P a_pv_p
    \in V,
    \quad
    f_{\bm a,E}
    \defeq
    \sum_{p=1}^P a_pe_p.
\]
Let
\[
    \hat{f}_{\bm a,V}
    \defeq
    \left(
        f_{\bm a,V}(X_1),\ldots,
        f_{\bm a,V}(X_N)
    \right)
\]
and define
\(\hat{f}_{\bm a,E}\) analogously.
Because \(f_{\bm a,V}\in V\), its sample-evaluation vector satisfies
\[
    \hat{f}_{\bm a,V}
    \in
    \bm V_{\bm X}.
\]
The interpolation condition therefore gives
\[
    \bm y
    =
    \hat{f}_{\bm a,V}
    +
    \hat{f}_{\bm a,E}.
\]
Applying
\(\Pi_{V, \bm X}^{\perp}\)
to both sides yields
\[
    \Pi_{V, \bm X}^{\perp}\bm y
    =
    \Pi_{V, \bm X}^{\perp}
    \hat{f}_{\bm a,E},
\]
since
\[
    \Pi_{V, \bm X}^{\perp}
    \hat{f}_{\bm a,V}
    =
    0.
\]
Orthogonal projection is contractive in the empirical norm \(\norm{\cdot}_N \defeq N^{-1/2}\norm{\cdot}_2\), and hence
\[
\norm{
        \Pi_{V, \bm X}^{\perp}\bm y
    }_N
    \leq
    \norm{
        \hat{f}_{\bm a,E}
    }_N
    \leq
    \norm{
        f_{\bm a,E}
    }_{L_\infty}.
\]
Finally,
\[
\begin{aligned}
    \norm{
        f_{\bm a,E}
    }_{L_\infty}
    &=
    \norm{
        \sum_{p=1}^P a_pe_p
    }_{L_\infty}
    \\
    &\leq
    \sum_{p=1}^P
    \abs{a_p}
    \norm{e_p}_{L_\infty}
    \\
    &\leq
    M_\sigma\delta
    \sum_{p=1}^P
    \abs{a_p}
    \\
    &=
    M_\sigma\delta
    \norm{\bm a}_1.
\end{aligned}
\]
Combining the preceding inequalities gives
\[
    \norm{
        \Pi_{V, \bm X}^{\perp}\bm y
    }_N
    \leq
    M_\sigma\delta
    \norm{\bm a}_1.
\]
\end{proof}

\begin{lemma}[Empirical isotropy of sparse-parity labels] \label{lemma:parity-empirical-isotropy} 
Let 
\( 
    X_1,\ldots,X_N \overset{\mathrm{i.i.d.}}{\sim} \operatorname{Unif} \left( \{\pm1\}^D \right) 
\).
For each \( S\in\binom{[D]}k\), 
define the sample label vector 
\[ 
    \bm y_S \defeq \left( \chi_S(X_1),\ldots,\chi_S(X_N) \right) \in \{\pm1\}^N, 
\] and let 
\[ 
    K_{\bm X} \defeq \frac{1}{\binom Dk} \sum_{S\in\binom{[D]}k} \bm y_S\bm y_S^\top. 
\] 
Then there exists a constant \(C_k>0\), depending only on \(k\), such that, for every \(0<\zeta<1\), with probability at least \(1-\zeta\), \[ \boxed{ \norm{ K_{\bm X}-I_N }_{\mathrm{op}} \leq C_k \frac{ N \log\left[\left( \frac{2N^2}{\zeta} \right) \right]^{k/2} }{ \sqrt{\binom Dk} }. } \] 
\end{lemma}

\begin{proof} The entries of \(K_{\bm X}\) are 
\[ 
    (K_{\bm X})_{ij} = \frac{1}{\binom Dk} \sum_{S\in\binom{[D]}k} \chi_S(X_i)\chi_S(X_j). 
\] 
Since every parity takes values in \(\{\pm1\}\), 
\(
    (K_{\bm X})_{ii} = 1 
\) for every $i \in [N]$.
Fix distinct \(i,j\in[N]\), and define
\[ 
    U_\ell^{ij} \defeq X_{i,\ell}X_{j,\ell}, \qquad \ell\in[D]. 
\] 
Because \(X_i\) and \(X_j\) are independent uniform Rademacher vectors, the variables \[ U_1^{ij},\ldots,U_D^{ij} \] are independent Rademacher random variables. 
Moreover, \[ 
    (K_{\bm X})_{ij} = \frac{1}{\binom Dk} \sum_{S\in\binom{[D]}k} \prod_{\ell\in S} U_\ell^{ij}. 
\] 
Thus, each off-diagonal entry is a homogeneous Rademacher polynomial of degree \(k\). The degree-\(k\) Walsh monomials are orthogonal. 
Therefore, 
\[  
    \mathbb E \left[ (K_{\bm X})_{ij}^2 \right] 
    = \frac{1}{\binom Dk^2} \sum_{S,T\in\binom{[D]}k} \mathbb E \left[ \prod_{\ell\in S}U_\ell^{ij} \prod_{m\in T}U_m^{ij} \right] 
    = \frac{1}{\binom Dk}.  
\] 
Hence \( \norm{ (K_{\bm X})_{ij} }_{L_2} = \frac{1}{ \sqrt{\binom Dk} }. \)

By the Bonami--Beckner hypercontractive inequality, applied to homogeneous degree-\(k\) Rademacher polynomials \citep{odonnellBasicsHypercontractivity2014}, for every \(q\geq2\), 
\[ 
    \norm{ (K_{\bm X})_{ij} }_{L_q} \leq (q-1)^{k/2} \norm{ (K_{\bm X})_{ij} }_{L_2}. 
\] 
It follows that 
\[ 
    \norm{ (K_{\bm X})_{ij} }_{L_q} \leq \frac{ (q-1)^{k/2} }{ \sqrt{\binom Dk} }. 
\] 
Set 
\[ 
    q \defeq \max \left\{ 2, \log \left( \frac{2N^2}{\zeta} \right) \right\}. 
\] 
Markov's inequality gives 
\[ 
    \begin{aligned} \mathbb P \left[ \abs{ (K_{\bm X})_{ij} } > \frac{ e(q-1)^{k/2} }{ \sqrt{\binom Dk} } \right] &\leq e^{-q}. \end{aligned} 
\] 
A union bound over all ordered pairs \(i\neq j\) therefore gives, with probability at least \(1-\zeta\), 
\[ 
    \max_{i\neq j} \abs{ (K_{\bm X})_{ij} } \leq C_k \frac{ [\log\left( \frac{2N^2}{\zeta} \right)]^{k/2} }{ \sqrt{\binom Dk} }. 
\] 
Since \(K_{\bm X}-I_N\) is symmetric and has zero diagonal, its operator norm is bounded by its maximum absolute row sum:
\[ 
\begin{aligned} 
    \norm{ K_{\bm X}-I_N }_{\mathrm{op}} 
    &\leq \max_{i\in[N]} \sum_{\substack{ j\in[N]\\
    j\neq i }} \abs{ (K_{\bm X})_{ij} } \\
    &\leq (N-1) \max_{i\neq j} \abs{ (K_{\bm X})_{ij} }. 
\end{aligned} 
\] 
Combining the preceding estimates proves 
\[ 
    \norm{ K_{\bm X}-I_N }_{\mathrm{op}} \leq C_k \frac{ N [\log \left( \frac{2N^2}{\zeta} \right)]^{k/2} }{ \sqrt{\binom Dk} }. 
\]
\end{proof}
Define
\[
    \eta_{D,N,\zeta}
    \defeq
    C_k
    \frac{
        N
        \left[
            \log
            \left(
                \frac{2N^2}{\zeta}
            \right)
        \right]^{k/2}
    }{
        \sqrt{\binom Dk}
    }, \quad \text{which }\quad = \tilde{O}_k\left(\frac{N}{D^{k/2}}\right)\ \text{for fixed $k,\gamma$}
\]
\begin{proposition}[Readout-norm cost of premature interpolation]
\label{prop:parity-interpolation-cost}
Let
\[
    X_1,\ldots,X_N
    \overset{\mathrm{i.i.d.}}{\sim}
    \operatorname{Unif}
    \left(
        \{\pm1\}^D
    \right),
\]
and let
\[
    S
    \sim
    \operatorname{Unif}
    \binom{[D]}k
\]
be independent of the initial representation conditional on the sample
inputs. Suppose the initial activity family \( \mathcal Z \)
admits an \(R\)-dimensional approximation at resolution \(\delta>0\). Assume that \(V\) is selected independently of \(S\), conditional on
the sample inputs,
\[
    S \perp (\mathcal{Z}, V) \mid \bm{X}.
\]
If
\[
    R<N
    \qquad\text{and}\qquad
    \eta_{D,N,\zeta}<1,
\]
then, with probability at least \(1-\zeta\) over the sample inputs, a
fraction at least
\[
    \frac{\mu}{2-\mu},
    \qquad
    \mu
    \defeq
    \left(
        1-\eta_{D,N,\zeta}
    \right)
    \left(
        1-\frac RN
    \right),
\]
of the \(k\)-sparse parity targets have the following property: any readout
that interpolates the target on the sample satisfies
\[
    \boxed{
        \norm{\bm a}_1
        \geq
        \frac{1}{
            M_\sigma\delta
        }
        \sqrt{
            \frac{\mu}{2}
        }.
    }
\]
\end{proposition}

\begin{proof}
By \cref{lemma:parity-empirical-isotropy}, with probability at least
\(1-\zeta\),
\[
    \norm{
        K_{\bm X}-I_N
    }_{\mathrm{op}}
    \leq
    \eta_{D,N,\zeta}.
\]
We condition on this event for the remainder of the proof. It implies
\[
    K_{\bm X}
    \succeq
    \left(
        1-\eta_{D,N,\zeta}
    \right)
    I_N.
\]
Because the target support \(S\) is independent of \(V\), conditional
on the sample inputs, we have
\[
\begin{aligned}
    \mathbb E_S
    \left[
        \norm{
            \Pi_{V,\bm X}^{\perp}\bm y_S
        }_2^2
        \,\middle|\,
        V,\bm X
    \right]
    &=
    \operatorname{tr}
    \left(
        \Pi_{V,\bm X}^{\perp}
        K_{\bm X}
    \right)
    \\
    &\geq
    \left(
        1-\eta_{D,N,\zeta}
    \right)
    \operatorname{tr}
    \left(
        \Pi_{V,\bm X}^{\perp}
    \right).
\end{aligned}
\]
Since
\(
    \dim V_{\bm X}
    \leq
    \dim V
    \leq
    R,
\)
we have
\[
    \operatorname{tr}
    \left(
        \Pi_{V,\bm X}^{\perp}
    \right)
    \geq
    N-R.
\]
Therefore,
\[
    \mathbb E_S
    \left[
        \frac1N
        \norm{
            \Pi_{V,\bm X}^{\perp}\bm y_S
        }_2^2
        \,\middle|\,
        V,\bm X
    \right]
    \geq
    \mu,
\]
where
\(
    \mu
    \defeq
    \left(
        1-\eta_{D,N,\zeta}
    \right)
    \left(
        1-\frac RN
    \right)
\).

Set
\[
    Z_S
    \defeq
    \frac1N
    \norm{
        \Pi_{V,\bm X}^{\perp}\bm y_S
    }_2^2.
\]
Since
\(\norm{\bm y_S}_2^2=N\)
and orthogonal projections are contractive,
\(
    0
    \leq
    Z_S
    \leq
    1.
\)

Let
\[
    q
    \defeq
    \mathbb P_S
    \left[
        Z_S
        \geq
        \frac{\mu}{2}
        \,\middle|\,
        V,\bm X
    \right].
\]
Using \(Z_S\leq1\), we obtain
\[
\begin{aligned}
    \mu
    \leq
    \mathbb E_S[Z_S\mid V,\bm X]
    &\leq
    q
    +
    (1-q)\frac{\mu}{2}.
\end{aligned}
\]
Rearranging gives
\[
    q
    \geq
    \frac{\mu}{2-\mu}.
\]

Hence, for at least a fraction
\(\mu/(2-\mu)\)
of the parity targets,
\[
    \frac1{\sqrt N}
    \norm{
        \Pi_{V,\bm X}^{\perp}\bm y_S
    }_2
    \geq
    \sqrt{\frac{\mu}{2}}.
\]
For each such target,
\cref{lemma:residual-interpolation}
implies that every interpolating readout satisfies
\[
\begin{aligned}
    \norm{\bm a}_1
    &\geq
    \frac{
        \norm{
            \Pi_{V,\bm X}^{\perp}\bm y_S
        }_2
    }{
        \sqrt N\,M_\sigma\delta
    }
    \\
    &\geq
    \frac1{
        M_\sigma\delta
    }
    \sqrt{
        \frac{\mu}{2}
    }.
\end{aligned}
\]
\end{proof}

\begin{corollary}[Fixed-target interpolation obstruction]
\label{cor:fixed-parity-interpolation-cost}
Suppose the assumptions of
\cref{prop:parity-interpolation-cost} hold. Assume additionally that the
joint law of the initial activity family, its
\(R\)-dimensional approximating space, and the input sample is invariant
under coordinate permutations:
\[
    \left(
        T_\pi\mathcal Z,
        T_\pi V,
        \pi\bm X
    \right)
    \overset{d}{=}
    \left(
        \mathcal Z,
        V,
        \bm X
    \right)
\]
for every permutation \(\pi\) of \([D]\), and suppose \(R<N\) and
\(\eta_{D,N,\zeta}<1\). Then, for every fixed
\(S^\star\in\binom{[D]}k\), with probability at least
\[
    (1-\zeta)
    \frac{\mu}{2-\mu}
\]
over the task-independent initialisation and input sample, every readout
that interpolates \(\chi_{S^\star}\) satisfies
\[
    \boxed{
        \norm{\bm a}_1
        \geq
        \frac{1}{
            M_\sigma\delta
        }
        \sqrt{
            \frac{\mu}{2}
        }.
    }
\]
\end{corollary}

\begin{remark}[Low-width interpolation requires norm inflation] Suppose 
\[ 
    \frac RN \longrightarrow 0, \qquad \eta_{D,N,\zeta} \longrightarrow 0, \qquad \zeta \longrightarrow 0.
\] 
Then, for a uniformly random \(k\)-sparse parity, with probability \(1-o(1)\) over the target and sample, every interpolating readout satisfies 
\[ 
    \norm{\bm a}_1 \geq \frac{1-o(1)}{ M_\sigma\delta }. 
\] 
Under the permutation-invariance assumption of \cref{cor:fixed-parity-interpolation-cost}, the same conclusion holds for every fixed \(k\)-sparse parity, with probability \(1-o(1)\) over the task-independent initialisation and sample. At the transfer-compatible resolution 
\[ 
    \delta_N \asymp \frac1{N\sqrt D},
\] 
this gives, for bounded activations with dimension-independent \(M_\sigma\),
\[ 
    \norm{\bm a}_1 = \Omega(N\sqrt D). 
\] Along the parity frontier 
\[ 
    NR = \widetilde\Theta_k \left( D^{k-1} \right), 
\] 
the conditions \(R/N\to0\) and \(\eta_{D,N,\zeta}\to0\) hold, up to logarithmic factors, throughout the window \[ D^{(k-1)/2} \ll N \ll D^{k/2}. \] Thus, on the low-width side of the crossover and within the empirical isotropy regime, exact interpolation requires readout norm \(\Omega(N\sqrt D)\) with probability \(1-o(1)\).
\end{remark}

\end{document}